\documentclass[11pt]{article}

\usepackage[final]{acl}

\usepackage{times}
\usepackage{latexsym}

\usepackage[T1]{fontenc}

\usepackage[utf8]{inputenc}

\usepackage{microtype}

\usepackage{inconsolata}

\usepackage{algorithmic}
\usepackage{algorithm}
\usepackage{graphicx}
\usepackage{amsthm}

\usepackage{amsmath,amsfonts,bm}

\def\eqref#1{equation~\ref{#1}}

\def\1{\bm{1}}

\DeclareMathAlphabet{\mathsfit}{\encodingdefault}{\sfdefault}{m}{sl}
\SetMathAlphabet{\mathsfit}{bold}{\encodingdefault}{\sfdefault}{bx}{n}

\usepackage[capitalize]{cleveref}
\theoremstyle{remark}

\newtheorem{theorem}{Theorem}[section]
\newtheorem{lemma}{Lemma}[section]

\newtheorem{assumption}{Assumption}[section]
\newtheorem{proposition}{Proposition}[section]
\crefname{theorem}{Thm.}{Thms.}
\Crefname{theorem}{Theorem}{Theorems}
\crefname{proposition}{Prop.}{Props.}
\Crefname{proposition}{Proposition}{Propositions}
\crefname{lemma}{Lem.}{Lems.}
\Crefname{lemma}{Lemma}{Lemmas}
\crefname{corollary}{Cor.}{Cors.}
\Crefname{corollary}{Corollary}{Corollaries}
\crefname{definition}{Def.}{Defs.}
\Crefname{definition}{Definition}{Definitions}
\crefname{assumption}{Assump.}{Assumps.}
\Crefname{assumption}{Assumption}{Assumptions}
\crefname{remark}{Rem.}{Rems.}
\Crefname{remark}{Remark}{Remarks}
\crefname{algorithm}{Alg.}{Algs.}
\Crefname{algorithm}{Algorithm}{Algorithms}
\def\method{AdaLeZO}

\usepackage{xcolor}
\usepackage{booktabs}
\usepackage{multirow}
\usepackage{subcaption}
\usepackage{amssymb}
\usepackage{colortbl}

\newcommand{\fei}[1]{#1}

\title{Universally Empowering Zeroth-Order Optimization via Adaptive Layer-wise Sampling}

\author{
  \textbf{Fei~Wang\textsuperscript{1,2}},
  \textbf{Li~Shen\textsuperscript{3,5,*}},
  \textbf{Liang~Ding\textsuperscript{4}},
  \textbf{Chao~Xue\textsuperscript{2}},
  \textbf{Ye~Liu\textsuperscript{1}},
  \textbf{Changxing~Ding\textsuperscript{1,}\thanks{Corresponding author.}}
\\
\\
  \textsuperscript{1}South China University of Technology,
  \textsuperscript{2}JD Explore Academy,
\\
  \textsuperscript{3}Shenzhen Campus of Sun Yat-sen University,
  \textsuperscript{4}University of Sydney,
\\
  \textsuperscript{5}Center for AI Theoretical Foundation and Systems, Shenzhen Loop Area Institute
\\
  \small{
    \href{mailto:ft_feiw@mail.scut.edu.cn}{ft\_feiw@mail.scut.edu.cn}, \href{mailto:chxding@scut.edu.cn}{chxding@scut.edu.cn}
  }
}

\begin{document}
\maketitle

\begin{abstract}
Zeroth-Order optimization presents a promising memory-efficient paradigm for fine-tuning Large Language Models by relying solely on forward passes. 
However, its practical adoption is severely constrained by slow wall-clock convergence and high estimation variance. 
In this work, we dissect the runtime characteristics of ZO algorithms and identify a critical system bottleneck where the generation of perturbations and parameter updates accounts for over 40\% of the training latency. 
We argue that the standard uniform exploration strategy is fundamentally flawed as it fails to account for the heterogeneous sensitivity of layers in deep networks, resulting in computationally wasteful blind searches. 
To address this structural mismatch, we propose \textbf{\method{}}, an \textbf{Ada}ptive \textbf{L}ayer-wis\textbf{e} \textbf{ZO} optimization framework. 
By formulating the layer selection process as a non-stationary Multi-Armed Bandit problem, \method{} dynamically allocates the limited perturbation budget to the most sensitive parameters.
We further introduce an Inverse Probability Weighting mechanism based on sampling with replacement, which guarantees unbiased gradient estimation while effectively acting as a temporal denoiser to reduce variance. 
Extensive experiments on LLaMA and OPT models ranging from 6.7B to 30B parameters demonstrate that \method{} achieves $1.7\times$ to $3.0\times$ wall-clock acceleration compared to state-of-the-art methods. Crucially, \method{} functions as a universal plug-and-play module that seamlessly enhances the efficiency of existing ZO optimizers without incurring additional memory overhead.~\footnote{\fei{The official implementation is available at \url{https://github.com/WangFei-2019/UnifiedZO}.}}
\end{abstract}
\section{Introduction}
\label{sec:intro}

Large Language Models (LLMs)~\citep{bai2023qwen,achiam2023gpt,bi2024deepseek,grattafiori2024llama,cai2026joyai} have demonstrated exceptional generalization capabilities across a broad spectrum of natural language processing tasks~\citep{liang2024controllable, zhu2024multilingual}. 
To adapt these general-purpose models to specialized downstream domains, Full Fine-Tuning~(FFT) remains the gold standard for achieving optimal performance~\citep{commonIT,rao-etal-2025-apt}. 
Nevertheless, FFT imposes prohibitive memory requirements as it necessitates the storage of optimizer states and gradient histories~\citep{kingma2015adam}.
While Parameter-Efficient Fine-Tuning methods~\citep{zhao2024galore,hu2022lora,li2021prefix}, such as LoRA~\citep{hu2022lora} and Prefix-Tuning~\citep{li2021prefix}, significantly reduce the number of trainable parameters, they still rely on backpropagation. Consequently, these methods require the transient storage of massive intermediate activations proportional to the network depth, which renders the fine-tuning of billion-scale models on consumer-grade hardware computationally infeasible.

\begin{figure*}[t!]
    \centering
    \includegraphics[width=0.9\textwidth]{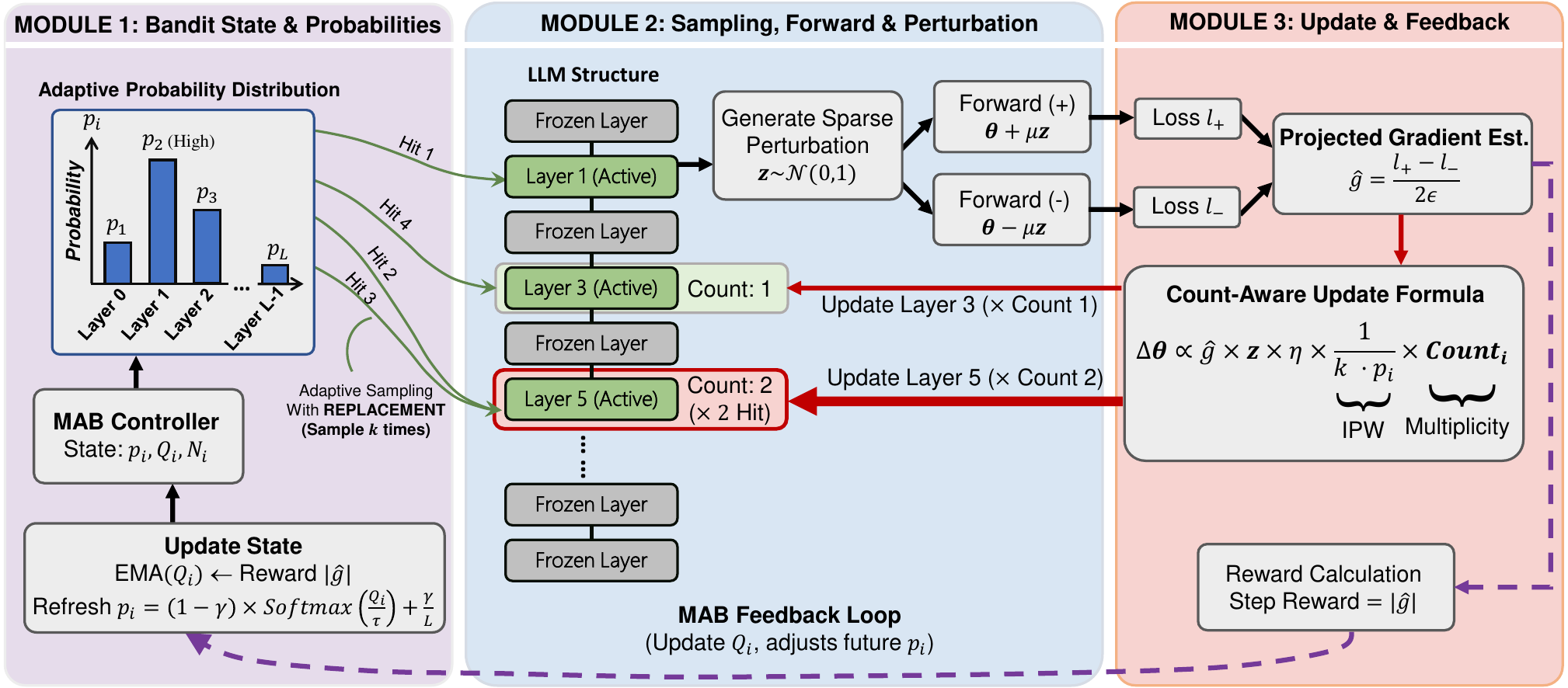}
    \caption{
    The \method{} workflow. \method{} overcomes the computational inefficiency of uniform ZO exploration by employing the MAB framework to allocate sparse perturbations to sensitive layers adaptively. Key modules include: (1) dynamic layer selection guided by real-time MAB statistics, (2) adaptive sampling with replacement to concentrate the perturbation budget on critical layers, and (3) a count-aware IPW update formula that ensures unbiased gradient estimation and enables efficient, feedback-driven optimization.}
    \label{fig:overview}
\end{figure*}

To circumvent the memory bottleneck intrinsic to backpropagation, Zeroth-Order~(ZO) optimization has recently garnered renewed interest~\citep{chen2024deepzero,wang2024simultaneous,zhang2024revisiting}. By estimating gradients through finite differences using only two forward passes, ZO methods such as MeZO~\citep{malladi2023fine} successfully compress the memory footprint of training to inference levels. Despite this breakthrough, standard ZO optimization is plagued by the ``curse of dimensionality,'' where the variance of gradient estimates scales linearly with the parameter size, leading to sluggish convergence and instability~\citep{chen2025enhancing,sun2025tezo}. Prior research has primarily focused on algorithmic refinements, such as introducing momentum or subspace constraints, to mitigate high variance~\citep{chen2019zo,jiang2024zo,chen2025enhancing,yu2025zeroth,zhao2025secondorder}. However, these approaches often overlook the significant wall-clock latency induced by high-dimensional operations.

In this work, we revisit ZO optimization from two complementary perspectives to uncover the root causes of its inefficiency. From a systems perspective, we observe that the operations required for perturbation generation and parameter updates incur a linear time complexity. Our breakdown analysis on an OPT-6.7B model reveals that these operations constitute nearly half of the per-step training time, creating a substantial linear bottleneck that limits scalability. From an optimization perspective, we identify a phenomenon we term ``Policy Blindness.'' While gradient information in LLMs is heterogeneously distributed across layers~\citep{wang2025layer}, standard ZO methods employ isotropic exploration strategies that treat all parameters equally. This mismatch results in a squandering of computational resources on insensitive layers that contribute minimal learning signals.

To address these dual challenges of computational redundancy and optimization blindness, we propose \textbf{Ada}ptive \textbf{L}ayer-wise \textbf{Z}eroth-\textbf{O}rder optimization (\method{}). Unlike heuristic approaches that rely on static priors, \method{} formalizes the optimization process as a layer selection problem within a Multi-Armed Bandit~(MAB) framework. This formulation allows the optimizer to act as a temporal filter, which integrates noisy instantaneous rewards to reveal the underlying sensitivity structure of the model. By dynamically concentrating the perturbation budget on the most critical layers, \method{} achieves sparse and efficient updates. To ensure theoretical rigor, we design a gradient estimator using Inverse Probability Weighting~(IPW) with replacement. This estimator ensures unbiasedness with respect to the full-parameter gradient while significantly reducing the estimation variance through importance sampling. The overall workflow of \method{} is depicted in \Cref{fig:overview}.

We empirically validate the effectiveness of our framework on models ranging from 6.7B to 30B parameters, including the LLaMA-3.1. The results demonstrate that \method{} delivers substantial wall-clock speedups of $1.7\times$ to $3.0\times$ over baseline methods while maintaining or surpassing competitive accuracy. Furthermore, \method{} exhibits strong versatility as a plug-and-play accelerator that consistently improves the performance of advanced ZO variants, such as LoZO~\citep{chen2025enhancing} and HiZOO~\citep{zhao2025secondorder}. In summary, our main contributions are as follows:
\begin{enumerate}
    \item We identify the linear cost of perturbation and update operations as a major bottleneck in the ZO optimizer and reveal the structural mismatch between uniform ZO exploration and the intrinsic layer-wise sparsity of gradients.
    \item We introduce \method{}, a novel framework that leverages a multi-armed bandit strategy to adaptively allocate perturbations, enabling efficient sparse optimization with theoretical unbiasedness guarantees.
    \item Extensive experiments confirm that \method{} provides significant wall-clock acceleration and universal compatibility with existing ZO algorithms, offering a scalable solution for memory-constrained LLM fine-tuning.
\end{enumerate}

\section{Related Work}
This section surveys the evolution of memory-efficient fine-tuning, focusing on the transition from parameter-efficient methods to ZO optimization and the integration of adaptive mechanisms.

\paragraph{Zeroth-Order Optimization for LLMs.}
While Parameter-Efficient Fine-Tuning methods like LoRA~\citep{hu2022lora} and Prefix-Tuning~\citep{li2021prefix} reduce trainable parameters, they remain constrained by the ``memory wall'' due to the storage of intermediate activations for backpropagation~\citep{dettmers2023qlora,liu2022p}. 
ZO optimization has emerged as a powerful alternative, enabling LLM fine-tuning with inference-level memory footprints by estimating gradients via forward pass differences~\citep{malladi2023fine,zhang2024revisiting,chen2024deepzero}. 
However, standard ZO methods (e.g., MeZO) suffer from the ``curse of dimensionality,'' where gradient estimation variance scales linearly with parameter size, leading to slow convergence~\citep{liu2020primer}. 
Although variants like Zo-Adamu~\citep{jiang2024zo} and MeZO-SVRG~\citep{gautam2024variance} introduce momentum or variance reduction to smooth the optimization trajectory, they largely overlook the linear wall-clock latency imposed by full-parameter perturbations. 
\fei{Recently, QuZO~\citep{zhou2025quzo} and MaZO~\citep{zhang2025mazo} have sought to further accelerate training via quantization, multi-task masking, or kernel transformations. However, these methods primarily optimize the temporal dimension or numerical precision. In contrast, \method{} is orthogonally designed to optimize the \textit{spatial dimension} by sparsifying $\mathcal{O}(d)$ memory reads and writes, acting as a plug-and-play module that can synergize seamlessly with these advancements for simultaneous spatial-temporal acceleration.}

\paragraph{Subspace Exploration and Structural Priors.}
To mitigate the high variance of ZO estimates, a growing body of work exploits the intrinsic low-dimensional structure of LLMs. 
Methods such as LoZO~\citep{chen2025enhancing}, SubZero~\citep{yu2025zeroth}, and TeZO~\citep{sun2025tezo} constrain perturbations to low-rank matrices or tensor decompositions, effectively reducing the search space. 
Others, including HiZOO~\citep{zhao2025secondorder} and LOREN~\citep{seung2025low}, leverage approximate second-order information or gradient priors to guide update directions. 
Despite the theoretical appeal, these approaches often rely on static structural assumptions (e.g., fixed rank) or incur expensive auxiliary computations that negate wall-clock speedups. 
Unlike these, which impose rigid constraints, \method{} focuses on \textit{dynamic structure discovery}, autonomously identifying sensitive layers during training without pre-defined priors.

\begin{figure*}[t!]
    \centering
    \begin{subfigure}[b]{0.48\textwidth}
    \includegraphics[width=\textwidth]{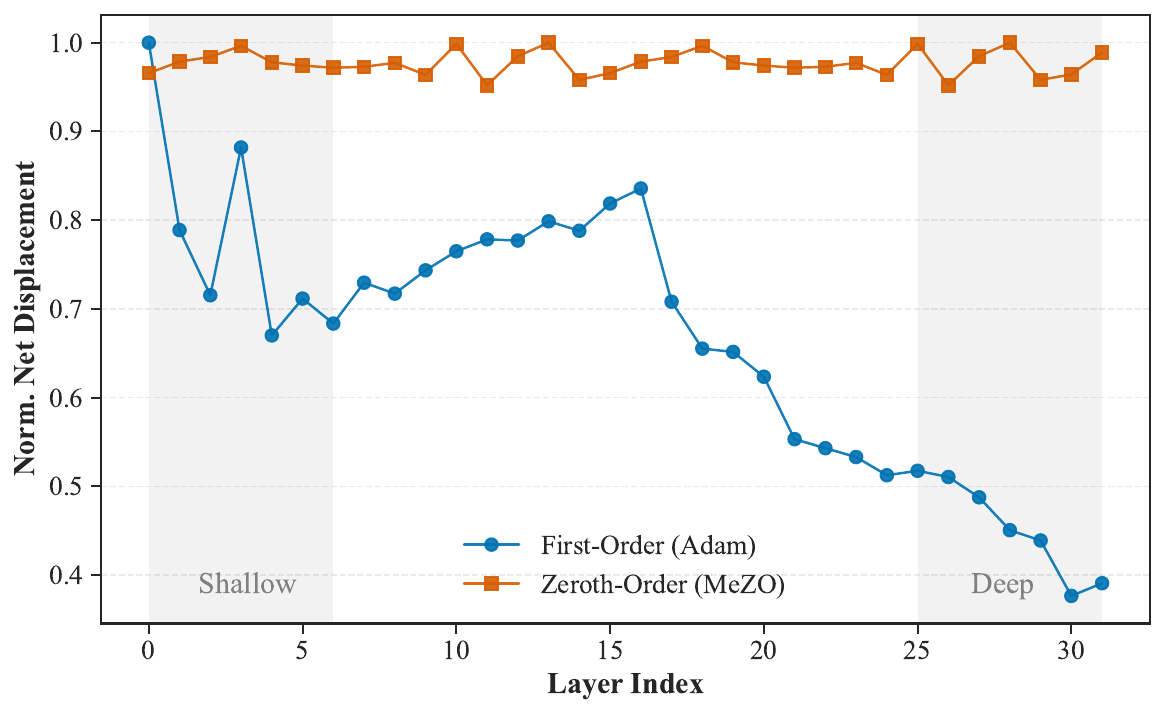}
        \caption{Layer-wise Net Displacement}
        \label{fig:zo_displacement}
    \end{subfigure}
    \hfill
    \begin{subfigure}[b]{0.48\textwidth}
        \includegraphics[width=\textwidth]{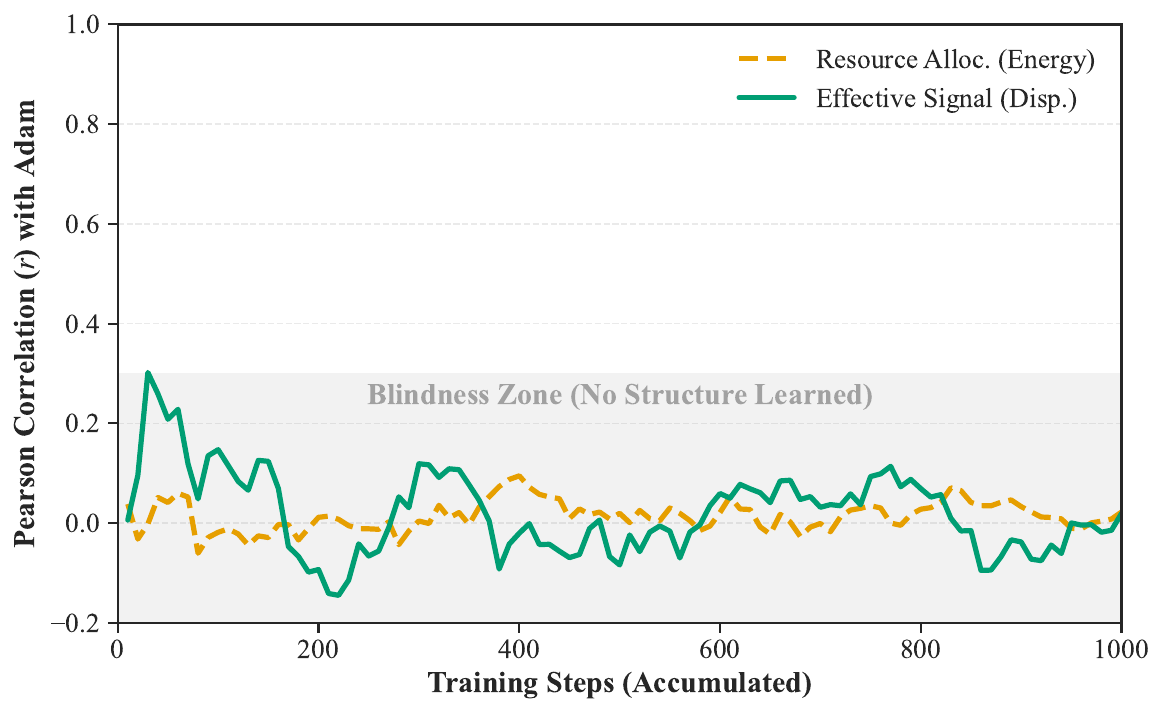}
        \caption{Pearson Correlation between ZO and FO Gradient Norm}
        \label{fig:pearson_zo_fo}
    \end{subfigure}
    \caption{
        \textbf{Empirical demonstration of Policy Blindness.} 
        We contrast the optimization dynamics of MeZO against those of Adam on OPT-6.7B.
        \textbf{(a) Layer-wise Net Displacement.} While the Adam exhibits distinct layer-wise heterogeneity by prioritizing updates on shallow layers, MeZO maintains a uniform update profile. This indicates that standard ZO methods squander the computational budget on insensitive parameters.
        \textbf{(b) Correlation Evolution.} The Pearson correlation between MeZO's cumulative updates and the Oracle gradient norm remains consistently low ($r < 0.2$). This confirms that isotropic perturbation fails to recover the intrinsic sensitivity structure of the model, resulting in a ``blind' random walk.
    }
    \label{fig:zo_policy}
\end{figure*}

\paragraph{Adaptive Resource Allocation.}
Dynamic sparsity has proven effective in post-training pruning~\citep{frantar2023sparsegpt,sun2024a}, where redundant weights are removed based on activation or Hessian sensitivity. 
\fei{In the context of First-Order (FO) fine-tuning, layer-wise and module-wise importance sampling techniques, such as LISA~\citep{pan2024lisa} and MISA~\citep{liu2025misa}, have demonstrated substantial efficiency gains. However, these methods fundamentally rely on exact FO gradients or intermediate activation caching to evaluate structural sensitivity.
Under strict ZO memory constraints, such FO priors are physically unavailable.}
In parallel, dynamic sampling and MAB frameworks have been successfully applied to sequential decision-making tasks in NLP, such as data selection and hyperparameter tuning~\citep{bouneffouf2025multi,lin2023bandit,ceritli2024study,rao2025dynamicsamplingadaptsiterative}. 
\method{} bridges these domains by integrating MAB into continuous ZO optimization. Crucially, it executes adaptive layer-wise importance sampling \textit{without} FO priors, relying solely on highly noisy, forward-pass scalar loss feedback. 
By treating layer selection as a bandit problem, \method{} realizes \textit{adaptive update sparsity}: it dynamically concentrates the perturbation budget on the most sensitive layers. 
This design resolves the ``blindness'' of uniform ZO exploration, ensuring theoretical unbiasedness while significantly reducing computational overhead.
\section{Rethinking Zeroth-Order Optimization}
\label{sec:rethink}
Despite the memory efficiency of ZO optimization, its widespread adoption is hindered by slow convergence and high variance. To elucidate the root causes of the limitations, we dissect ZO algorithms through the lenses of optimization dynamics, system latency, and signal fidelity. These empirical insights provide the motivational foundation for \method{}.

\subsection{Optimization Dynamics: Policy Blindness}
\label{subsec:rethink_dynamics}

A fundamental inefficiency in current ZO methods arises from a structural mismatch between the uniform exploration strategy and the heterogeneous sensitivity of LLM layers. We term this phenomenon \textit{Policy Blindness}.
To quantify this, we compare the parameter evolution of a FO oracle (Adam~\citep{kingma2015adam}) against MeZO using \textit{Layer-wise Net Displacement} ($||\sum \Delta \theta||_2$) as a proxy for effective learning signals.
As illustrated in \Cref{fig:zo_policy}, the FO oracle exhibits distinct layer-wise heterogeneity where shallow and middle layers accumulate significant updates while deeper layers remain relatively static~\citep{clark2019bert, aghajanyan2021intrinsic}.
In stark contrast, MeZO employs an isotropic perturbation strategy that distributes the computational budget uniformly across the entire parameter space.
This homogeneous allocation squanders resources on insensitive parameters that contribute negligibly to loss reduction, thereby injecting excessively high-dimensional noise that impedes convergence.

\begin{figure}[t!]
    \centering
    \includegraphics[width=\columnwidth]{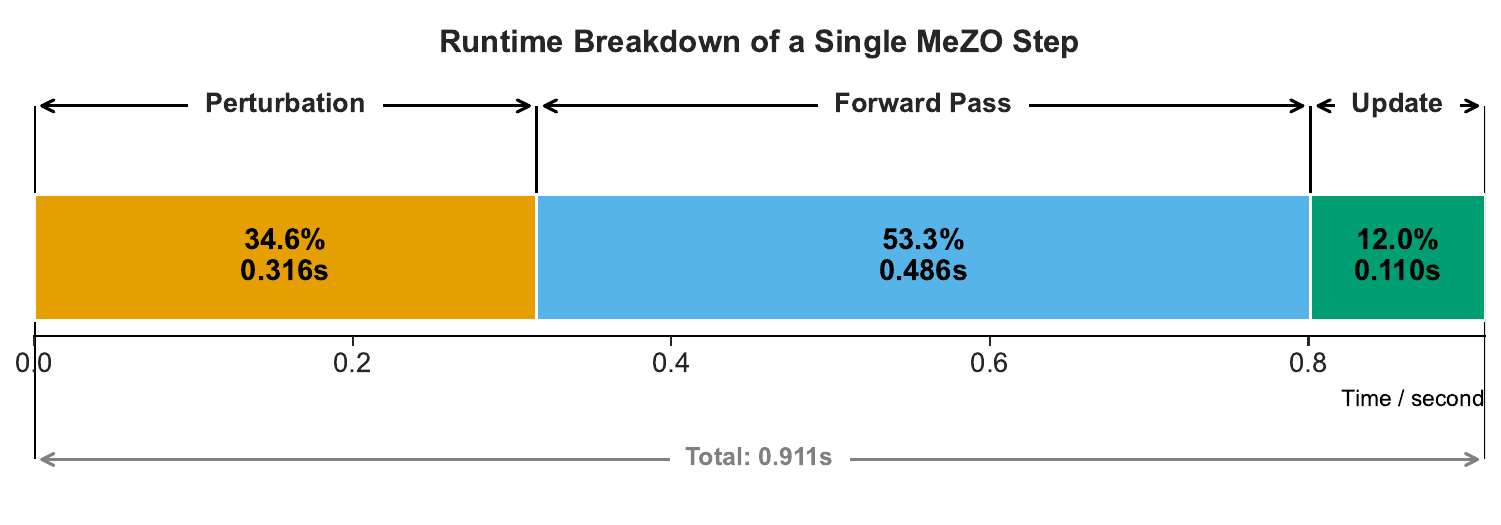}
    \caption{
    Time breakdown of a single training step in MeZO. The operations for perturbation generation and parameter updates constitute approximately 46\% of the total step time, comparable to the forward pass time.
    }
    \label{fig:mezo_runtime}
\end{figure}

\subsection{System Efficiency: Linear Bottleneck}
\label{subsec:rethink_efficiency}

Beyond optimization inefficiency, full-parameter perturbations introduce a critical system bottleneck. 
Runtime profiling on an OPT-6.7B model, as shown in \Cref{fig:mezo_runtime}, reveals that perturbation generation and parameter updates account for approximately 46\% of the total training time per step.
This cost scales linearly with the parameter dimension $d$, rendering it a dominant factor for billion-scale models.
Furthermore, classical ZO theory indicates that gradient estimation variance also scales with $d$~\citep{nesterov2017random}.
These observations suggest that restricting perturbations to a sparse subset of sensitive layers can simultaneously eliminate this linear wall-clock overhead and reduce estimation variance, provided that the active layers are correctly identified.

\subsection{Signal Fidelity: Validity of Feedback}
\label{subsec:rethink_fidelity}

A prerequisite for adaptive layer selection is the existence of a reliable signal within the noisy ZO estimates to guide the search.
We investigate whether the magnitude of the ZO gradient estimate $|\hat{g}|$ serves as a valid proxy for the true gradient norm $\|\nabla \mathcal{L}\|_F$.
Our fidelity analysis in \Cref{fig:signal_fidelity} yields two key findings.
First, at the \textit{micro-level}, despite the high variance intrinsic to random projection, we observe a significant positive Spearman correlation ($\rho \approx 0.48$) between the ZO estimate and the ground-truth gradient norm.
Second, at the \textit{macro-level}, binned analysis reveals a strictly monotonic relationship between the expected ZO magnitude and the true gradient norm.
This statistical consistency implies that while individual ZO samples are noisy, their expectation faithfully preserves the relative ranking of layer importance.
Consequently, sequential decision algorithms such as Multi-Armed Bandits can effectively exploit this property to recover the true sensitivity structure through temporal aggregation.

\section{Methodology}
\label{sec:method}

In this section, we first revisit the standard paradigm of ZO optimization for fine-tuning LLMs and analyze its inherent computational bottlenecks. Subsequently, we propose the \method{} framework. We formulate layer-wise selection as a non-stationary MAB problem and introduce a sparse gradient estimator based on IPW with replacement, achieving efficient, sparse, and low-variance optimization.

\subsection{Preliminaries: The Linear Bottleneck of ZO Optimization}

Consider an LLM parameterized by $\theta \in \mathbb{R}^d$. Our objective is to optimize the loss function $\mathcal{L}(\theta) = \mathbb{E}_{\mathcal{D}}[f(\theta; x, y)]$. To circumvent the prohibitive memory cost of storing intermediate activations required by backpropagation, standard ZO methods, such as MeZO~\cite{malladi2023fine}, employ the Simultaneous Perturbation Stochastic Approximation (SPSA) algorithm to estimate gradients. At step $t$, the algorithm generates a random perturbation vector $z_t \sim \mathcal{N}(0, I_d)$ drawn from a standard normal distribution and computes the gradient estimate via two forward passes:
\begin{equation}
\label{eq:standard_zo}
\hat{g}_t^{\text{ZO}} = \frac{\mathcal{L}(\theta_t + \mu z_t) - \mathcal{L}(\theta_t - \mu z_t)}{2\mu} z_t,
\end{equation}
where $\mu$ denotes the smoothing parameter (perturbation radius). The parameters are updated following $\theta_{t+1} = \theta_t - \eta \hat{g}_t^{\text{ZO}}$.

Although MeZO successfully eliminates the memory bottleneck, \Cref{eq:standard_zo} reveals a fundamental flaw regarding computational efficiency: $z_t$ is \textbf{dense}. Consequently, every optimization step necessitates sampling, adding, and updating all $d$ parameters. As analyzed in \Cref{subsec:rethink_efficiency}, in the context of billion-scale models ($d \ge 7\text{B}$), this full-parameter operation incurs a linear time complexity of $O(d)$, accounting for over 40\% of the total training time. Furthermore, applying uniform perturbation across the entire high-dimensional parameter space introduces significant estimation variance, known as the ``curse of dimensionality''~\cite{liu2020primer, nesterov2017random}, which severely hampers convergence.

\subsection{Adaptive Layer Selection via Multi-Armed Bandit}

To mitigate the aforementioned computational redundancy and high variance, \method{} leverages the layer-wise heterogeneity of LLM gradients by modeling the selection of trainable layers as a Multi-Armed Bandit (MAB) problem. We partition the model parameters into $L$ groups (layers), denoted as $\{\theta^{(1)}, \dots, \theta^{(L)}\}$. Our goal is to dynamically learn a sampling policy $\pi_t$ that allocates the limited perturbation budget to the layers most sensitive to the loss function.

\paragraph{Reward Definition.}
We require a reward signal to quantify the contribution of a specific layer to the optimization process. We define the immediate reward $R_t$ at step $t$ as the magnitude of the estimated scalar gradient. Intuitively, if a perturbation induces a significant change in the loss, it indicates that the optimization direction lies in a steep region of the loss landscape, thereby offering higher optimization value:
\begin{equation}
R_t = \left| \frac{\mathcal{L}(\theta_t + \mu z_t) - \mathcal{L}(\theta_t - \mu z_t)}{2\mu} \right|.
\end{equation}
This scalar serves as a proxy for the effectiveness of the current step. As analyzed in \Cref{app:theory}, the optimal sampling probability is proportional to the gradient norm, i.e., $p_t(l) \propto \|\nabla^{(l)} \mathcal{L}\|$, validating $R_t$ as an ideal proxy signal.

\paragraph{Value Estimation (EMA).}
Since the sensitivity of layers varies dynamically during training (i.e., the environment is non-stationary), we employ an Exponential Moving Average (EMA) to maintain the value estimate $Q_t(l)$ for each layer. For every layer $l$ in the selected active set $\mathcal{I}_t$ at step $t$, the value is updated as:
\begin{equation}
Q_{t+1}(l) = (1 - \alpha) Q_t(l) + \alpha R_t,
\end{equation}
where $\alpha$ is the learning rate factor. For unselected layers, $Q$ remains unchanged. This design allows the algorithm to smooth out historical noise while rapidly adapting to the distribution shift in layer importance.

\paragraph{Policy with Exploration-Exploitation.}
To balance the exploitation of highly sensitive layers with the exploration of under-sampled ones, we compute the sampling probability distribution $p_t$ based on the current $Q_t$ values. For the $l$-th layer:
\begin{equation}
p_t(l) = (1 - \gamma) \cdot \underbrace{\text{Softmax}(Q_t(l) / \tau)}_{\text{Exploitation}} + \gamma \cdot \underbrace{\frac{1}{L}}_{\text{Exploration}},
\label{eq:prob}
\end{equation}
where $\tau$ is the temperature coefficient controlling distribution smoothness, and $\gamma \in [0, 1]$ is the mixing coefficient. The inclusion of the uniform distribution term $\frac{1}{L}$ is critical; it guarantees a non-zero lower bound probability $p_t(l) \ge \gamma/L$, preventing the permanent ``starvation'' of layers, which is a necessary condition for the unbiasedness of the subsequent estimator.

\subsection{Sparse Gradient Estimation with IPW}

Based on the probability distribution $p_t$, we design a sparse gradient estimator utilizing Sampling with Replacement and IPW.

\paragraph{Sampling Mechanism.}
Unlike previous methods that employ sampling without replacement (Top-$k$), \method{} performs $K = \max(1, \lfloor \rho L \rfloor)$ independent draws \textbf{with replacement} based on $p_t$ at step $t$, where $\rho \in (0, 1]$ is the sampling ratio. Let $n_{t,l}$ denote the number of times layer $l$ is selected (multiplicity). If $n_{t,l} > 0$, the layer is marked as active, forming the set $\mathcal{I}_t = \{l \mid n_{t,l} > 0\}$.

\paragraph{Sparse Perturbation.}
We generate Gaussian noise only for the active layers to construct a sparse perturbation vector $\tilde{z}_t$:
\begin{equation}
\tilde{z}_t^{(l)} = \begin{cases} 
\mathcal{N}(0, I_{d_l}) & \text{if } l \in \mathcal{I}_t, \\
0 & \text{otherwise}.
\end{cases}
\end{equation}
Since $|\mathcal{I}_t| \le K \ll L$, this operation significantly reduces the overhead of perturbation generation and parameter updates from $O(d)$ to $O(\rho d)$.

\paragraph{Count-Aware IPW Estimator.}
Directly using sparse perturbation introduces bias. To address this, we apply importance-sampling reweighting, incorporating the counting property of sampling with replacement, the \method{} gradient estimate for layer $l$ is defined as:
\begin{equation}
\label{eq:adalezo_grad_update}
    \hat{g}_t^{\text{Ada},(l)} = \hat{g}_{\text{scalar}} \cdot w_{t,l} \cdot n_{t,l} \cdot \tilde{z}_t^{(l)},
\end{equation}
where the projected gradient scalar is $\hat{g}_{\text{scalar}} \triangleq \frac{\mathcal{L}(\theta_t + \mu \tilde{z}_t) - \mathcal{L}(\theta_t - \mu \tilde{z}_t)}{2\mu}$, and the IPW weight is $w_{t,l} \triangleq \min\left(\frac{1}{K p_t(l)}, C_{\text{clip}}\right)$.
In this formulation, the multiplicity $n_{t,l}$ leverages the statistical nature of sampling with replacement, ensuring that important layers selected multiple times ($n_{t,l} > 1$) receive larger update steps to automatically intensify optimization on sensitive parameters. Simultaneously, the IPW weight $w_{t,l}$ functions as a clipped inverse probability weight which, as proved in \Cref{app:theory}, converges to an unbiased estimator of the full-parameter Gaussian smoothed gradient, i.e., $\mathbb{E}[\hat{g}_t^{\text{Ada}}] = \nabla \mathcal{L}_\mu(\theta_t)$, as the clipping threshold $C_{\text{clip}} \to \infty$. Furthermore, regarding variance reduction, while standard IPW ensures unbiasedness, extremely small sampling probabilities $p_t(l)$ can induce numerical instability; thus, the clipping threshold $C_{\text{clip}}$ is introduced to establish a bias-variance trade-off that significantly reduces estimation variance in practical scenarios.

The final parameter update is performed strictly on active layers: $\theta_{t+1}^{(l)} = \theta_t^{(l)} - \eta \hat{g}_t^{\text{Ada},(l)}$. The complete algorithmic procedure is detailed in \Cref{alg:adalezo}. We further provide a rigorous theoretical proof in \Cref{app:theory} demonstrating that \method{} maintains a convergence rate of $O(1/\sqrt{T})$.
\section{Experiments}
\label{sec:exp}

\begin{figure}[htbp]
    \centering
    \vspace{-4pt}
    \includegraphics[width=\columnwidth]{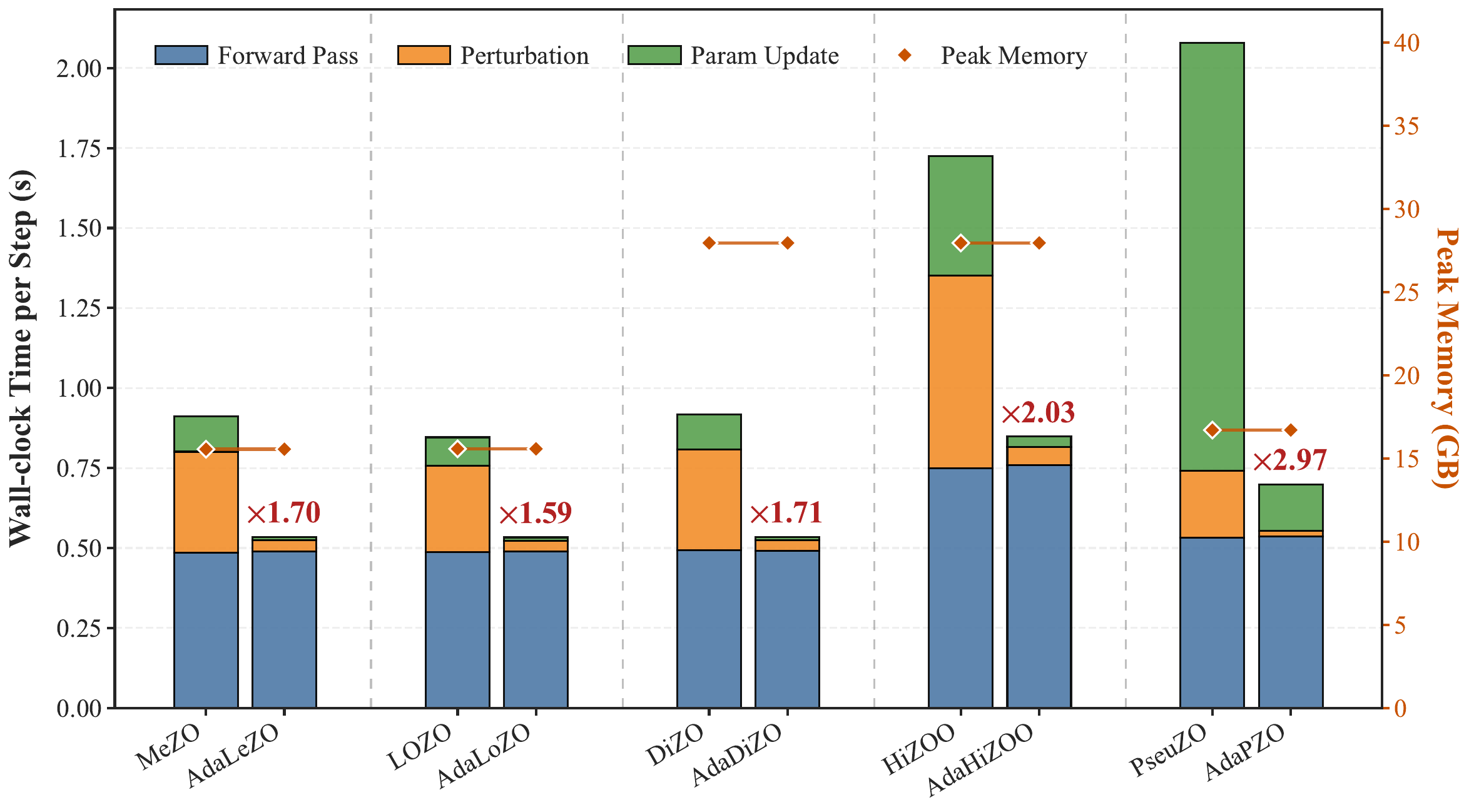}
    \caption{
    Breakdown of wall-clock time per training step and peak memory consumption across different ZO optimizations. The stacked bars represent the time cost of forward pass, perturbation generation, and parameter update, while the orange diamonds indicate peak memory usage. Our proposed Ada- methods significantly compress the overhead of perturbation and updates. 
    }
    \label{fig:latency_memory}
\end{figure}

\begin{table*}[t]
\centering
\caption{\textbf{Main Results on LLaMA Series.} We compare AdaLeZO against the Zero-shot baseline and MeZO (Full-parameter ZO). The results are averaged over 3 random seeds. The best performance between ZO methods is marked in \textbf{bold}.}
\label{tab:llama_results}
\resizebox{\textwidth}{!}{%
\begin{tabular}{l|ccccccc|cc|cc|c}
\toprule
\multicolumn{1}{c|}{\multirow{2}{*}{\textbf{Method}}} & \multicolumn{7}{c|}{\textbf{Classification}} & \multicolumn{2}{c|}{\textbf{Multiple Choice}} & \multicolumn{2}{c|}{\textbf{Generation}} & \multirow{2}{*}{\textbf{AVG.}} \\
\cmidrule(lr){2-8} \cmidrule(lr){9-10} \cmidrule(lr){11-12}
\multicolumn{1}{c|}{} & SST-2 & RTE & CB & BoolQ & WSC & WIC & MultiRC & Copa & ReCoRD & SQuAD & DROP & \\ \midrule 
\multicolumn{13}{c}{\textit{\textbf{LLaMA-2-7B}}} \\ 
\midrule
Zero-shot & 58.14 & 63.18 & 32.14 & 70.60 & 36.54 & 50.16 & 45.50 & 78.00 & 80.70 & 55.47 & 20.20 & 53.69 \\
MeZO & 93.92{\scriptsize$\pm$0.90} & 63.66{\scriptsize$\pm$2.08} & 68.45{\scriptsize$\pm$2.73} & \textbf{78.60}{\scriptsize$\pm$0.52} & \textbf{60.58}{\scriptsize$\pm$3.33} & 61.91{\scriptsize$\pm$3.70} & \textbf{71.77}{\scriptsize$\pm$3.44} & \textbf{84.33}{\scriptsize$\pm$1.53} & 81.03{\scriptsize$\pm$0.70} & \textbf{88.64}{\scriptsize$\pm$0.63} & 40.95{\scriptsize$\pm$0.69} & 72.17{\scriptsize$\pm$1.84} \\
AdaLeZO & \textbf{94.00}{\scriptsize$\pm$0.13} & \textbf{66.19}{\scriptsize$\pm$2.71} & \textbf{69.05}{\scriptsize$\pm$2.73} & 76.30{\scriptsize$\pm$1.85} & 60.26{\scriptsize$\pm$3.89} & \textbf{62.91}{\scriptsize$\pm$2.51} & 70.43{\scriptsize$\pm$2.97} & 84.00{\scriptsize$\pm$1.00} & \textbf{81.40}{\scriptsize$\pm$0.82} & 88.49{\scriptsize$\pm$0.30} & \textbf{41.98}{\scriptsize$\pm$1.27} & \textbf{72.27}{\scriptsize$\pm$\textbf{1.83}} \\ 
\midrule
\multicolumn{13}{c}{\textit{\textbf{LLaMA-3.1-8B}}} \\ 
\midrule
Zero-shot & 59.75 & 46.57 & 44.64 & 76.30 & 50.78 & 59.62 & 62.30 & 64.88 & 83.70 & 85.00 & 28.46 & 60.18 \\
MeZO & \textbf{93.81}{\scriptsize$\pm$0.41} & \textbf{75.45}{\scriptsize$\pm$2.50} & 70.24{\scriptsize$\pm$1.03} & 80.73{\scriptsize$\pm$0.80} & \textbf{62.18}{\scriptsize$\pm$1.11} & 58.03{\scriptsize$\pm$1.94} & 74.87{\scriptsize$\pm$1.35} & 89.67{\scriptsize$\pm$0.58} & 83.90{\scriptsize$\pm$1.00} & 89.37{\scriptsize$\pm$1.26} & 60.78{\scriptsize$\pm$1.23} & 76.28{\scriptsize$\pm$\textbf{1.20}} \\
AdaLeZO & 92.93{\scriptsize$\pm$0.92} & 71.24{\scriptsize$\pm$5.26} & \textbf{72.02}{\scriptsize$\pm$2.73} & \textbf{81.20}{\scriptsize$\pm$0.96} & 61.22{\scriptsize$\pm$2.22} & \textbf{59.51}{\scriptsize$\pm$0.80} & \textbf{80.17}{\scriptsize$\pm$0.76} & \textbf{90.33}{\scriptsize$\pm$0.58} & \textbf{84.70}{\scriptsize$\pm$1.05} & \textbf{89.74}{\scriptsize$\pm$0.78} & \textbf{61.87}{\scriptsize$\pm$0.31} & \textbf{76.81}{\scriptsize$\pm$1.49} \\ 
\bottomrule
\end{tabular}%
}
\end{table*}
\begin{table*}[ht]
\centering
\caption{\textbf{Universal Effectiveness on OPT-6.7b.} We report the accuracy (\%) and standard deviation across three random seeds. We compare various ZO methods with their adaptive counterparts powered by our \method{} framework. The best result in each pair (e.g., MeZO vs. \method{}) is marked in \textbf{bold}.}
\label{tab:multi_methods_result}
\resizebox{\textwidth}{!}{%
\begin{tabular}{l|ccccccc|cc|cc|c}
\toprule
\multicolumn{1}{c|}{\multirow{2}{*}{\textbf{Method}}} & \multicolumn{7}{c|}{\textbf{Classification}} & \multicolumn{2}{c|}{\textbf{Multiple Choice}} & \multicolumn{2}{c|}{\textbf{Generation}} & \multirow{2}{*}{\textbf{AVG.}} \\
\cmidrule(lr){2-8} \cmidrule(lr){9-10} \cmidrule(lr){11-12}
\multicolumn{1}{c|}{} & SST-2 & RTE & CB & BoolQ & WSC & WIC & MultiRC & Copa & ReCoRD & SQuAD & DROP & \\ \midrule
Zero-shot & 61.24 & 54.87 & 50.00 & 63.10 & 37.50 & 51.25 & 44.50 & 82.00 & 76.00 & 36.48 & 17.70 & 52.24 \\
ICL & 84.29 & 65.70 & 57.14 & 70.40 & 51.92 & 53.61 & 50.20 & 81.00 & 76.90 & 74.30 & 27.75 & 63.02 \\
FT (Upper Bound) & 92.78 & 78.34 & 94.64 & 73.90 & 59.62 & 51.25 & 77.50 & 80.00 & 75.30 & 85.28 & 28.53 & 72.47 \\ 
\midrule
MeZO & 93.54{\scriptsize$\pm$0.48} & 65.46{\scriptsize$\pm$2.08} & 69.05{\scriptsize$\pm$1.03} & 67.13{\scriptsize$\pm$0.31} & \textbf{56.09}{\scriptsize$\pm$5.47} & 58.78{\scriptsize$\pm$1.10} & \textbf{66.10}{\scriptsize$\pm$2.51} & \textbf{81.33}{\scriptsize$\pm$1.53} & 78.03{\scriptsize$\pm$0.60} & \textbf{81.52}{\scriptsize$\pm$1.43} & 28.79{\scriptsize$\pm$1.04} & 67.80{\scriptsize$\pm$1.60} \\
\textbf{\method{}} & \textbf{93.58}{\scriptsize$\pm$0.23} & \textbf{67.39}{\scriptsize$\pm$1.50} & \textbf{70.24}{\scriptsize$\pm$2.73} & \textbf{67.93}{\scriptsize$\pm$0.49} & 55.45{\scriptsize$\pm$4.74} & \textbf{59.61}{\scriptsize$\pm$1.16} & 63.13{\scriptsize$\pm$1.67} & 80.33{\scriptsize$\pm$2.52} & \textbf{79.07}{\scriptsize$\pm$0.49} & 80.65{\scriptsize$\pm$0.23} & \textbf{29.19}{\scriptsize$\pm$0.99} & \textbf{67.87}{\scriptsize$\pm$\textbf{1.52}} \\ 
\midrule
LOZO & \textbf{93.69}{\scriptsize$\pm$0.61} & \textbf{67.51}{\scriptsize$\pm$2.73} & 69.64{\scriptsize$\pm$1.79} & \textbf{69.53}{\scriptsize$\pm$0.38} & 50.64{\scriptsize$\pm$11.23} & 59.30{\scriptsize$\pm$3.76} & 59.87{\scriptsize$\pm$1.50} & 79.00{\scriptsize$\pm$1.00} & 78.83{\scriptsize$\pm$0.55} & \textbf{81.65}{\scriptsize$\pm$0.19} & \textbf{29.49}{\scriptsize$\pm$2.51} & 67.20{\scriptsize$\pm$2.39} \\
\textbf{AdaLoZO} & 93.27{\scriptsize$\pm$0.07} & 66.19{\scriptsize$\pm$1.27} & \textbf{72.02}{\scriptsize$\pm$2.73} & 67.33{\scriptsize$\pm$0.45} & \textbf{56.09}{\scriptsize$\pm$2.94} & \textbf{62.07}{\scriptsize$\pm$0.87} & \textbf{60.30}{\scriptsize$\pm$0.30} & \textbf{80.67}{\scriptsize$\pm$0.58} & \textbf{79.17}{\scriptsize$\pm$0.55} & 78.03{\scriptsize$\pm$1.59} & 28.76{\scriptsize$\pm$1.37} & \textbf{67.63}{\scriptsize$\pm$\textbf{1.16}} \\ 
\midrule
DiZO & 92.93{\scriptsize$\pm$0.13} & \textbf{67.03}{\scriptsize$\pm$1.63} & \textbf{69.64}{\scriptsize$\pm$3.09} & \textbf{68.77}{\scriptsize$\pm$2.79} & \textbf{60.90}{\scriptsize$\pm$4.54} & \textbf{62.00}{\scriptsize$\pm$1.31} & 62.03{\scriptsize$\pm$1.07} & 78.33{\scriptsize$\pm$2.31} & 78.13{\scriptsize$\pm$0.21} & \textbf{80.38}{\scriptsize$\pm$1.34} & 25.93{\scriptsize$\pm$1.65} & 67.83{\scriptsize$\pm$1.83} \\
\textbf{AdaDiZO} & \textbf{93.04}{\scriptsize$\pm$0.63} & 66.79{\scriptsize$\pm$0.63} & 69.05{\scriptsize$\pm$1.03} & 68.20{\scriptsize$\pm$0.46} & 59.30{\scriptsize$\pm$3.09} & 61.02{\scriptsize$\pm$1.28} & \textbf{62.97}{\scriptsize$\pm$0.75} & \textbf{82.33}{\scriptsize$\pm$2.52} & \textbf{78.57}{\scriptsize$\pm$0.21} & 79.59{\scriptsize$\pm$1.82} & \textbf{29.24}{\scriptsize$\pm$1.45} & \textbf{68.19}{\scriptsize$\pm$\textbf{1.26}} \\ 
\midrule
HiZOO & \textbf{93.43}{\scriptsize$\pm$0.48} & \textbf{65.46}{\scriptsize$\pm$2.40} & 69.05{\scriptsize$\pm$1.03} & \textbf{68.03}{\scriptsize$\pm$1.46} & 56.41{\scriptsize$\pm$5.80} & 59.25{\scriptsize$\pm$1.10} & \textbf{65.57}{\scriptsize$\pm$2.97} & \textbf{82.00}{\scriptsize$\pm$2.00} & 78.13{\scriptsize$\pm$0.29} & \textbf{81.96}{\scriptsize$\pm$1.35} & 27.81{\scriptsize$\pm$1.18} & \textbf{67.92}{\scriptsize$\pm$1.82} \\
\textbf{AdaHiZOO} & 92.85{\scriptsize$\pm$0.52} & 64.86{\scriptsize$\pm$1.85} & \textbf{71.43}{\scriptsize$\pm$1.79} & 67.17{\scriptsize$\pm$1.27} & \textbf{56.73}{\scriptsize$\pm$5.85} & \textbf{60.24}{\scriptsize$\pm$1.31} & 64.23{\scriptsize$\pm$0.15} & 80.33{\scriptsize$\pm$0.58} & \textbf{78.97}{\scriptsize$\pm$0.74} & 80.12{\scriptsize$\pm$1.88} & \textbf{28.75}{\scriptsize$\pm$0.75} & 67.79{\scriptsize$\pm$\textbf{1.52}} \\ 
\midrule
PseuZO & \textbf{93.77}{\scriptsize$\pm$0.07} & \textbf{65.22}{\scriptsize$\pm$2.76} & \textbf{71.43}{\scriptsize$\pm$1.79} & 67.30{\scriptsize$\pm$0.62} & \textbf{56.09}{\scriptsize$\pm$11.27} & \textbf{61.55}{\scriptsize$\pm$1.19} & 59.67{\scriptsize$\pm$1.55} & 79.00{\scriptsize$\pm$1.00} & \textbf{78.47}{\scriptsize$\pm$0.50} & 78.60{\scriptsize$\pm$1.28} & 26.08{\scriptsize$\pm$2.30} & 67.02{\scriptsize$\pm$2.21} \\
\textbf{AdaPZO} & 93.00{\scriptsize$\pm$0.40} & 64.50{\scriptsize$\pm$1.37} & \textbf{71.43}{\scriptsize$\pm$0.00} & \textbf{67.67}{\scriptsize$\pm$2.14} & 55.13{\scriptsize$\pm$5.80} & 61.34{\scriptsize$\pm$0.39} & \textbf{60.23}{\scriptsize$\pm$0.50} & \textbf{80.67}{\scriptsize$\pm$1.15} & \textbf{78.47}{\scriptsize$\pm$0.64} & \textbf{78.67}{\scriptsize$\pm$0.96} & \textbf{28.57}{\scriptsize$\pm$0.90} & \textbf{67.24}{\scriptsize$\pm$\textbf{1.30}} \\ 
\bottomrule
\end{tabular}%
}
\end{table*}

\subsection{Experimental Setup}
\label{sec:exp_setup}
We evaluate \method{} on LLaMA-2-7B~\citep{touvron2023llama}, LLaMA-3.1-8B~\citep{grattafiori2024llama}, and OPT 6.7B-30B~\citep{zhang2022opt} models across 11 downstream tasks. We compare against MeZO~\citep{malladi2023fine}, LoZO~\citep{chen2025enhancing}, DiZO~\citep{tan2025harmony}, HiZOO~\citep{zhao2025secondorder}, and PseuZO~\citep{yue2025pseuzo}. Detailed hyperparameters and baselines are listed in \Cref{app:implementation}.

\subsection{Efficiency Analysis: Breaking the Linear Barrier}
\label{subsec:efficiency}

We analyze the core bottleneck in ZO optimization: the linear complexity of perturbation operations. \Cref{fig:latency_memory} presents a comparative visualization of wall-clock latency and memory usage for \method{} versus standard ZO baselines.

\textbf{Wall-Clock Acceleration.} Standard MeZO suffers from structural inefficiency, with perturbation generation and parameter updates accounting for nearly half of each training step. By restricting these operations to a sparse set of active layers, \method{} effectively removes this bottleneck. \fei{To rigorously ensure fairness, we evaluate the absolute hardware throughput on a single NVIDIA A100 (40GB) GPU in BF16 precision. For OPT-6.7B (batch size 16, sequence length 256), \method{} achieves a throughput of 7,728 tokens/sec, significantly surpassing MeZO's 4,501 tokens/sec.} This translates to a genuine $1.7\times$ hardware-level speedup, confirming that the MAB bookkeeping overhead is negligible ($\ll 1\%$ of step time). Furthermore, when applied to PseuZO, the speedup reaches almost $3.0\times$ by circumventing costly projection steps. Detailed results in \Cref{tab:latency_breakdown} show that perturbation overhead is reduced to less than $10\%$ of step time.

\textbf{Memory Neutrality.} Notably, \method{} preserves the peak memory footprint of standard ZO methods, as the orange diamonds shown in \Cref{fig:latency_memory}. In contrast to FO approaches that require substantial amounts of optimizer state, \method{} operates entirely within inference-level memory constraints, ensuring compatibility with consumer hardware. 
As shown in \Cref{tab:speedup_details}, the efficiency improvements are consistent across sequence lengths, achieving up to $5.1\times$ speedup on short-sequence tasks where perturbation overhead is most significant.

\begin{figure*}[t!]
    \centering
    \begin{subfigure}[b]{0.24\textwidth}
        \centering
        \includegraphics[width=\textwidth]{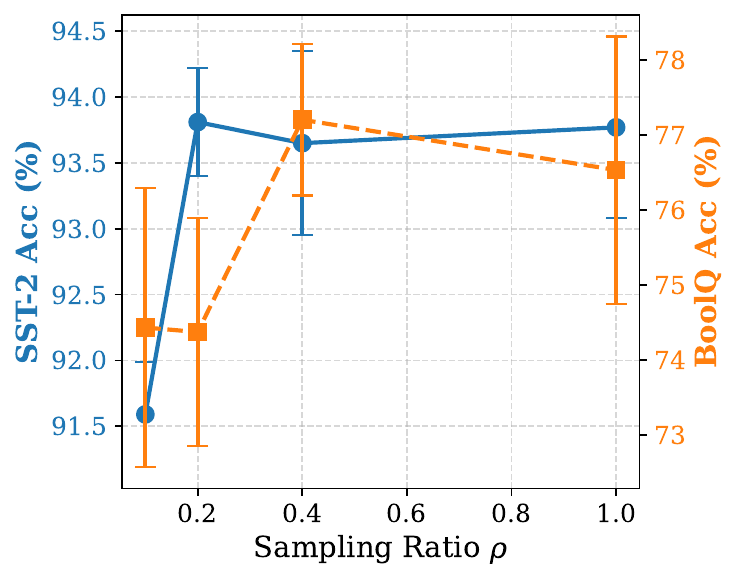}
        \caption{Effect of Sampling Ratio $\rho$}
        \label{fig:ablation_rho}
    \end{subfigure}
    \hfill
    \begin{subfigure}[b]{0.24\textwidth}
        \centering
        \includegraphics[width=\textwidth]{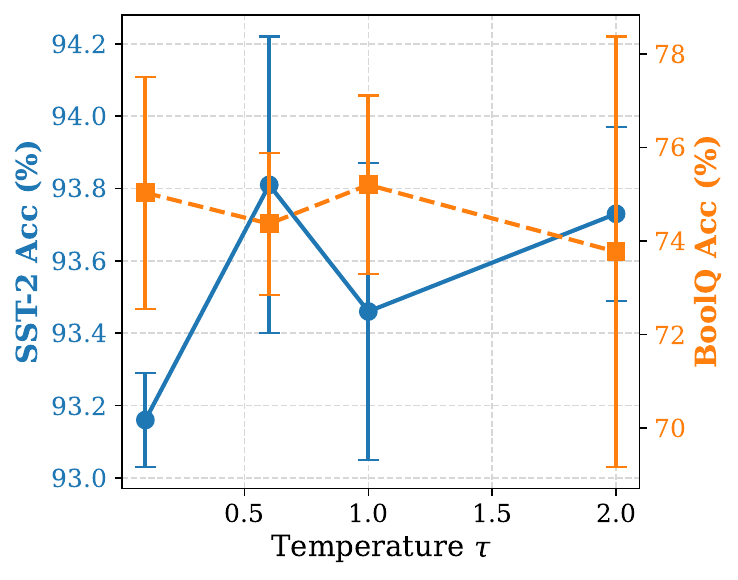}
        \caption{Effect of Temperature $\tau$}
        \label{fig:ablation_tau}
    \end{subfigure}
    \hfill
    \begin{subfigure}[b]{0.24\textwidth}
        \centering
        \includegraphics[width=\textwidth]{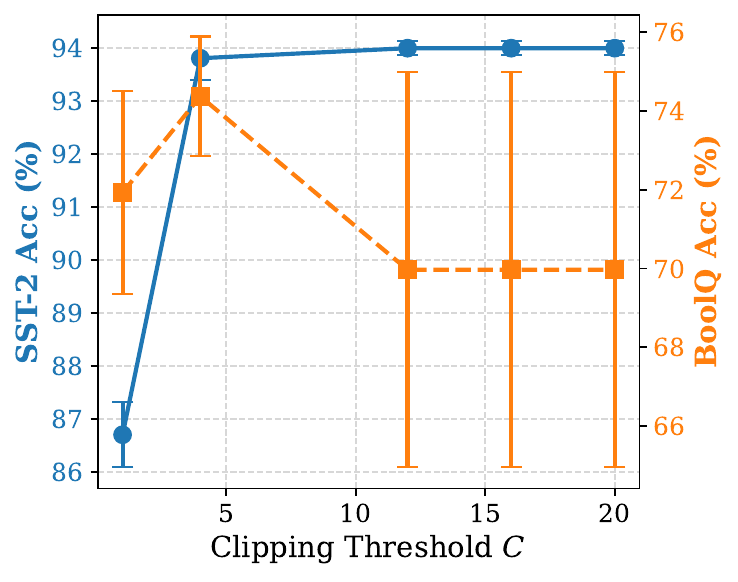}
        \caption{Effect of IPW Clipping $C$}
        \label{fig:ablation_clip}
    \end{subfigure}
    \hfill
    \begin{subfigure}[b]{0.24\textwidth}
        \centering
        \includegraphics[width=\textwidth]{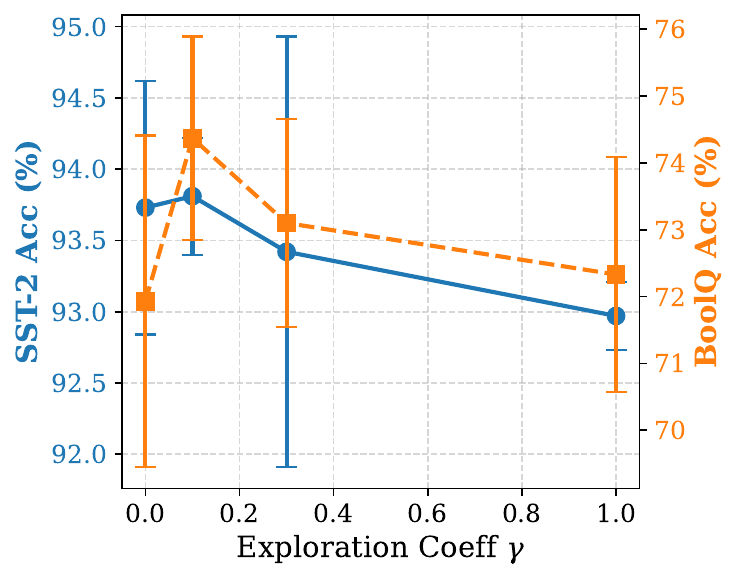}
        \caption{Effect of Exploration $\gamma$}
        \label{fig:ablation_gamma}
    \end{subfigure}
    
    \caption{\textbf{Ablation Studies.} We analyze the impact of four key hyperparameters on SST-2 (blue, left axis) and BoolQ (orange, right axis) performance. Error bars denote standard deviation across 3 seeds. Detail in \Cref{tab:ablation_full}.}
    \label{fig:ablation_all}
\end{figure*}

\subsection{Main Results on LLaMA Series}
\label{subsec:main_results}

\Cref{tab:llama_results} presents the performance of MeZO and \method{} on LLaMA-2-7B and LLaMA-3.1-8B models. \method{} demonstrates consistent superiority over the dense baseline.

\textbf{Superiority over Dense Updates.} \method{} achieves highly competitive performance against, and often surpasses, the full-parameter MeZO baseline. On challenging reasoning benchmarks such as DROP and SQuAD, \method{} delivers notable improvements, including a gain of 1.03\% on DROP with LLaMA-2. These results challenge the notion that updating more parameters necessarily leads to better outcomes, and demonstrate that adaptive sparse updates can effectively suppress the detrimental noise associated with high-dimensional dense perturbations.

\textbf{Stability and Robustness.} ZO optimization is often hindered by high variance. \method{} achieves lower standard deviations compared to MeZO, indicating greater stability. This improvement stems from the count-aware IPW estimator, which, as detailed in \Cref{sec:method}, reduces estimation variance by focusing updates on frequently sampled, high-sensitivity layers.

\subsection{Universality across ZO Optimizers}
\label{subsec:universality}

A defining feature of \method{} is its orthogonality to existing ZO enhancements. We integrate \method{} as a plug-and-play module into four representative baselines. 
As shown in \Cref{tab:multi_methods_result}, the \method{}-enhanced variants (denoted as Ada-) consistently improve or maintain the accuracy of their base methods while providing the significant speedups discussed in \Cref{subsec:efficiency}. 
For instance, AdaHiZOO not only accelerates HiZOO by $2.0\times$ but also improves average accuracy by $0.15\%$ on OPT-6.7B. This universality confirms that adaptive layer-wise sparsity captures a fundamental property of the loss landscape that is complementary to subspace constraints (LoZO) or gradient priors (HiZOO). 
Additionally, we demonstrate the scalability of \method{} on OPT-13B and OPT-30B models in \Cref{tab:multi_size_result}, where it continues to outperform MeZO, validating its effectiveness for large-scale fine-tuning.

\subsection{Ablation Study}
\label{sec:ablation}

We conduct comprehensive ablation studies with the LLaMA-2-7b model to evaluate the contributions of each \method{} component and hyperparameter sensitivity. Two representative tasks are selected: SST-2 (robust to noise) and BoolQ (sensitive to gradient estimation errors). Key results are shown in \Cref{fig:ablation_all}.

\paragraph{Necessity of Sparse Update}
\Cref{fig:ablation_all}(a) shows the effect of sampling ratio $\rho$. Performance improves markedly as $\rho$ increases from 0.1 to 0.2, suggesting that excessive sparsity may discard crucial gradient information. However, further increasing $\rho$ to 0.4 or 1.0 (full-parameter MeZO) yields diminishing or negative returns. This supports our hypothesis: \textbf{full-parameter perturbation is often suboptimal for zeroth-order optimization}. Given the ZO error bound $O(d/\sqrt{T})$, estimation variance scales with dimension $d$. By reducing the update space to $\rho d$, \method{} lowers variance and introduces implicit regularization, outperforming dense baselines. 
\fei{Furthermore, to definitively isolate the optimization benefits of our MAB-driven policy from the wall-clock speedups of raw sparsity, we demonstrate in \Cref{app:random_sparse} that \method{} consistently outperforms a uniform Random Sparse baseline (by 2.12\% on average), proving the necessity of adaptive layer selection.}

\paragraph{Variance Reduction is Critical}
\Cref{fig:ablation_all}(c) highlights the sensitivity to the IPW clipping threshold $C$. At $C=1$ (unweighted sparse updates), performance drops due to sampling bias. Conversely, a large $C$ ($\ge 12$) increases estimator variance, causing accuracy on BoolQ to collapse (from 74.37\% to 69.97\%), as extreme weights $1/(K p_l)$ amplify gradient variance. Setting $C=4$ achieves an optimal balance between \textbf{bias} and \textbf{variance}, which is crucial for high-dimensional ZO optimization.

\paragraph{Balancing Exploration and Exploitation}
\Cref{fig:ablation_all}(b) and (d) analyze the bandit strategy. A small temperature ($\tau=0.1$, nearly greedy) leads to premature convergence, while a large $\tau$ ($=2.0$, nearly uniform) fails to exploit gradient sensitivity; both underperform compared to $\tau=0.6$. The exploration coefficient $\gamma$ is also essential: adding slight uniform exploration ($\gamma=0.1$) boosts BoolQ performance by 2.4\% over pure Softmax ($\gamma=0$), confirming that preventing ``layer starvation'' is vital for estimator coverage and convergence in non-stationary bandit settings.

\section{Conclusion}
\label{sec:conclusion}
In this work, we identify the linear computational cost of perturbation and parameter updates as a fundamental bottleneck in zeroth-order (ZO) large language model (LLM) fine-tuning. To address this limitation, we introduce \method{}, a novel framework that integrates multi-armed bandit (MAB)-driven adaptive layer selection with a count-aware inverse probability weighting (IPW) estimator, enabling efficient, sparse, and unbiased gradient estimation. By dynamically concentrating perturbations on the most sensitive layers, \method{} effectively mitigates the curse of dimensionality and offers theoretical guarantees for convergence.
Extensive experiments on both LLaMA and OPT models demonstrate that \method{} achieves $1.7\times$ to $5.1\times$ wall-clock speedup over state-of-the-art baselines, without any loss in accuracy. Furthermore, the plug-and-play design of \method{} allows for seamless integration with existing ZO optimizers, providing a unified and scalable solution for memory-efficient and rapid LLM fine-tuning.
\section*{Limitations}
\label{sec:limitations}

Although \method{} delivers substantial improvements in wall-clock efficiency and convergence stability, several inherent limitations of the Zeroth-Order optimization paradigm remain.

\paragraph{Performance Gap with First-Order Methods.}
\method{} consistently surpasses standard zeroth-order baselines such as MeZO and narrows the gap with full fine-tuning. However, its performance does not fully reach that of first-order methods across all tasks. For example, in challenging reasoning benchmarks such as DROP, the stochastic nature of gradient estimation, even when enhanced by variance reduction techniques, results in a precision trade-off relative to backpropagation. Addressing this gap continues to be a fundamental challenge for the zeroth-order optimization community.

\paragraph{Scope of Application.}
The present evaluation is limited to supervised fine-tuning of large language models within the field of natural language processing. The effectiveness of adaptive layer-wise sampling for other areas, including multimodal large language models and reinforcement learning from human feedback, has not yet been empirically validated. Investigating these extensions represents an important direction for future research.

\section*{Acknowledgements}
This work was supported by the National Key R\&D Projects under Grant 2024YFC3307100; the National Natural Science Foundation of China under Grants 62076101, 62576364, and 62306118; the Guangdong Basic and Applied Basic Research Foundation under Grants 2024B1515020082, 2023A1515010007, 2026B1515020071, and 2026A1515010725; the Guangdong Provincial Key Laboratory of Human Digital Twin under Grant 2022B1212010004; the Shenzhen Basic Research Project (Natural Science Foundation) Basic Research Key Project under Grant JCYJ20241202124430041; the Fundamental Research Funds for the Central Universities under Grant 2025ZYGXZR054; the TCL Young Scholars Program; and the 2024 Tencent AI Lab Rhino-Bird Focused Research Program.


\bibliography{custom}

\newpage
\appendix
\crefalias{section}{appendix}
\begin{table}[ht]
\centering
\caption{\textbf{Hyperparameter Settings.} We list the common settings shared across all ZO methods and the specific hyperparameters for each baseline and \method{}.}
\label{tab:hyperparams}
\resizebox{\columnwidth}{!}{%
\begin{tabular}{llc}
\toprule
\textbf{Category} & \textbf{Hyperparameter} & \textbf{Value} \\ \midrule
\multicolumn{3}{c}{\textit{\textbf{Common Settings for All ZO Methods}}} \\ \cmidrule(lr){1-3}
\multirow{3}{*}{General} & Batch Size & 16 \\
 & Learning Rate & $\{1, 5, 10\} \times 10^{-7}$ \\
 & Perturbation Scale $\mu$ & $1\times 10^{-3}$ \\ \midrule
\multicolumn{3}{c}{\textit{\textbf{Method-Specific Settings}}} \\ \cmidrule(lr){1-3}
\textbf{LOZO} & Rank $r$ & 2 \\
 & Interval $v$ & 50 \\ \addlinespace
\textbf{HiZOO} & Estimate times $n$ & 1 \\
 & Smooth scale $\alpha$ & $1 \times 10^{-8}$ \\ \addlinespace
\textbf{DiZO} & Projection Update Cycle & 100 \\
 & Projection Iterations & 10 \\
 & Smoothing Scalar & 0.1 \\
 & Projection Step Size & 2 \\
 & Clip Range & 0.2 \\ \addlinespace
\textbf{PseuZO} & Sliding Window $L$ & 14 \\
 & Cycles / Epochs & 2 / 10 \\
 & Descent Formula & $\lambda(t) = \frac{\lambda_{\max}}{1 + 10 t}$ \\ \midrule
\multirow{5}{*}{\textbf{AdaLeZO} (Ours)} & Sampling Ratio $\rho$ & 0.2 \\
 & Temperature $\tau$ & 0.6 \\
 & Exploration Coeff. $\gamma$ & $\{0, 0.1\}$ \\
 & EMA Factor $\alpha$ & 0.1 \\
 & Clipping Threshold $C$ & $\{4, 16\}$ \\ \midrule
\multicolumn{3}{c}{\textit{\textbf{Fine-Tuning (Reference)}}} \\ \cmidrule(lr){1-3}
\multirow{3}{*}{FT (Adam)} & Batch Size & 8 \\
 & Learning Rate & $1 \times 10^{-5}$ \\
 & Scheduler & Linear \\ \bottomrule
\end{tabular}%
}
\end{table}

\section{Experiments Detial}
\label{app:implementation}
\subsection{Baseline}
We conducted comparative evaluations of \method{} against three baseline approaches: zero-shot, in-context learning~(ICL), fine-tuning~(FT), and MeZO~\citep{malladi2023fine}. The zero-shot approach assesses both pre-trained models without any fine-tuning, serving as a lower-bound performance baseline. Fine-tuning (FT) is referenced to indicate the performance of non-quantized models. To further investigate the generalizability of \method{}, we also compared it with four state-of-the-art MeZO-based zero-order optimization methods: LoZO~\citep{chen2025enhancing}, DiZO~\citep{tan2025harmony}, HiZOO~\citep{zhao2025secondorder}, and PseuZO~\citep{yue2025pseuzo}.

\subsection{Models, Dataset, and Metrics} 
Our experiments encompass two prominent model families: the OPT~\citep{zhang2022opt} family and the LLaMA~\citep{touvron2023llama,grattafiori2024llama} family. To ensure comprehensive coverage of model scales, we selected four models from the OPT~\citep{zhang2022opt} family: OPT-6.7B, OPT-13B, and OPT-30B. For contemporary relevance, we included LLaMA-2-7B~\citep{touvron2023llama} and LLaMA-3.1-8B~\citep{grattafiori2024llama} from the LLaMA family. 

For downstream evaluation, we employed eleven tasks commonly used in ZO optimization literature. These include seven classification tasks from the SuperGLUE~\citep{wallach2019superglue} benchmark: SST-2~\citep{socher2013recursive}, RTE~\citep{bentivogli2009fifth}, CB~\citep{de2019commitmentbank}, BoolQ~\citep{zhang2018record}, WIC~\citep{pilehvar2019wic}, WSC~\citep{levesque2012winograd}, and MultiRC~\citep{khashabi2018looking}; two multiple-choice tasks from SuperGLUE benchmark: Copa~\citep{roemmele2011choice} and ReCoRD~\citep{zhang2018record}; and two question-answering tasks: SQuAD~\citep{rajpurkar2016squad} and DROP~\citep{dua2019drop}, which we treat as generation tasks. For classification and multiple-choice tasks, we report accuracy; for generation tasks, we report the F1 score.

\subsection{Implementation}
For all ZO methods, we ensured fair comparison by adopting the optimal hyperparameter settings from their original publications and conducting a grid search over learning rates. 
Detailed hyperparameter configurations are provided in \Cref{tab:hyperparams}. For \method{}, we used the default parameters specified in \cref{tab:hyperparams} across all experiments, except in ablation studies where only the parameters under investigation were modified. \method{} employs the same learning rate as the baseline for comparison.

During training, all models were fine-tuned for 20,000 steps with checkpoints saved every 5,000 steps. We selected the checkpoint with the lowest validation loss for final evaluation on the test set. For each task, we used 1,000 samples for training and 500 for validation; when fewer than 1,000 samples were available, we used 100 for validation and the remainder for training. All available test samples were used for evaluation. To ensure statistical reliability, we repeated each experiment with three different random seeds and report the average performance metric~\citep{rao2022reproducibility}.

\begin{table}[t]
\centering
\caption{\textbf{Ablation Study on LLaMA-2-7b.} We report the average accuracy and standard deviation across 3 seeds. Default parameters are \underline{underlined}. \textbf{Bold} indicates the best performance. This table validates the contribution of each component in our Bandit-based strategy.}
\label{tab:ablation_full}
\resizebox{\columnwidth}{!}{%
\begin{tabular}{l|cc|cc}
\toprule
\textbf{Ablation Component} & \multicolumn{2}{c|}{\textbf{SST-2}} & \multicolumn{2}{c}{\textbf{BoolQ}} \\
(Parameter) & \textbf{Avg.} & \textbf{Std.} & \textbf{Avg.} & \textbf{Std.} \\ 
\midrule
\multicolumn{5}{l}{\textit{Sampling Ratio $\rho$}} \\
0.1 & 91.59 & 0.40 & 74.43 & 1.86 \\
\underline{0.2} & \textbf{93.81} & 0.41 & 74.37 & 1.52 \\
0.4 & 93.65 & 0.70 & \textbf{77.20} & 1.01 \\
1.0 (Full) & 93.77 & 0.69 & 76.53 & 1.78 \\
\midrule
\multicolumn{5}{l}{\textit{Temperature $\tau$}} \\
0.1 (Greedy) & 93.16 & 0.13 & \textbf{75.03} & 2.48 \\
\underline{0.6} & \textbf{93.81} & 0.41 & 74.37 & \textbf{1.52} \\
1.0 & 93.46 & 0.41 & 75.20 & 1.91 \\
2.0 (Uniform) & 93.73 & 0.24 & 73.77 & 4.60 \\
\midrule
\multicolumn{5}{l}{\textit{Clipping Threshold $C$}} \\
1 & 86.70 & 0.61 & 71.93 & 2.57 \\
\underline{4} & 93.81 & 0.41 & \textbf{74.37} & \textbf{1.52} \\
12 & \textbf{94.00} & 0.13 & 69.97 & 5.01 \\
16 & \textbf{94.00} & 0.13 & 69.97 & 5.01 \\
\midrule
\multicolumn{5}{l}{\textit{Exploration $\gamma$}} \\
0.0 (Pure Softmax) & 93.73 & 0.89 & 71.93 & 2.48 \\
\underline{0.1} & \textbf{93.81} & 0.41 & \textbf{74.37} & \textbf{1.52} \\
0.3 & 93.42 & 1.51 & 73.10 & 1.56 \\
1.0 (Uniform) & 92.97 & 0.24 & 72.33 & 1.76 \\
\midrule
\multicolumn{5}{l}{\textit{EMA Factor $\alpha$}} \\
\underline{0.1} & \textbf{93.81} & 0.41 & 74.37 & \textbf{1.52} \\
0.5 & 92.85 & 0.59 & \textbf{75.10} & 3.32 \\
\bottomrule
\end{tabular}%
}
\end{table} 
\section{Additional Ablation Studies}
\label{app:ablation_details}

In this section, we provide a granular analysis of the hyperparameters governing the Multi-Armed Bandit (MAB) mechanism in \method{}, specifically the Reward EMA factor $\alpha$, Temperature $\tau$, and Exploration coefficient $\gamma$. These parameters are critical for balancing the trade-off between plasticity and stability in the non-stationary optimization landscape of LLM fine-tuning. The full numerical comparisons are detailed in \Cref{tab:ablation_full}.

\subsection{Sensitivity to Bandit Hyperparameters}

\paragraph{Reward EMA Factor ($\alpha$).}
The EMA factor $\alpha$ controls how quickly the bandit forgets historical rewards. A larger $\alpha$ makes the policy more responsive to recent gradient estimates but also more susceptible to instantaneous noise.
Comparing the default $\alpha=0.1$ with a more aggressive update rate $\alpha=0.5$, we observe in \Cref{tab:ablation_full} that while $\alpha=0.5$ achieves comparable average accuracy, it introduces significant instability. Specifically, on the BoolQ dataset, the standard deviation for $\alpha=0.5$ is more than double that of $\alpha=0.1$ ($3.32$ vs. $1.52$). This result corroborates that in Zeroth-Order optimization, where gradient estimates naturally possess high variance, a lower $\alpha$ is preferable as it effectively smooths out the noise, providing a stable signal for layer importance.

\paragraph{Temperature ($\tau$).}
The temperature parameter $\tau$ modulates the sharpness of the Softmax distribution derived from the estimated Q-values. 
Our results indicate that extreme values are detrimental to performance. A low temperature ($\tau=0.1$) approximates a greedy strategy, which risks premature convergence to suboptimal layers and leads to high performance variance (Std $2.48$ on BoolQ). Conversely, a high temperature ($\tau=2.0$) approaches uniform sampling, diluting the benefits of adaptive selection. The choice of $\tau=0.6$ strikes an optimal balance, allowing the model to exploit sensitive layers while maintaining sufficient entropy in the sampling distribution.

\paragraph{Exploration Coefficient ($\gamma$).}
We further verify the necessity of the explicit exploration term $\gamma/L$ in \Cref{eq:prob}. The empirical results demonstrate that a small mixing coefficient $\gamma=0.1$ consistently outperforms the pure Softmax strategy ($\gamma=0$), particularly on challenging tasks. This confirms that ensuring a non-zero lower bound on sampling probabilities is crucial for preventing ``layer starvation'', which is a scenario where potentially useful layers are permanently ignored due to early stochastic fluctuations, thereby maintaining the asymptotic validity of the estimator.
\textbf{When $\gamma=1$, the method degenerates to random search, resulting in a significant performance decline on the SST-2 and BoolQ benchmarks. This observation indirectly corroborates the effectiveness of the proposed Multi-Armed Bandit strategy.
}

\fei{\section{Impact of the Multi-Armed Bandit Policy versus Random Sparsity}
\label{app:random_sparse}

To empirically validate the necessity of the proposed MAB-driven adaptive layer selection, we must isolate the optimization benefits of our policy from the wall-clock speedups provided by raw sparsity. To this end, we introduce a \textbf{Random Sparse} baseline that selects layers uniformly at random at each step, strictly utilizing the same sparsity budget as \method{} ($\rho=0.2$). 

We evaluate this baseline on the LLaMA-2-7B model across a 9-task subset of our evaluation suite. The results are summarized in \Cref{tab:random_sparse}.

\begin{table*}[ht]
\centering
\caption{\fei{Comparison between \method{} and the Random Sparse baseline on LLaMA-2-7B. Both methods operate under identical sparsity budgets ($\rho=0.2$). The MAB policy yields a +2.12\% average improvement.}}
\label{tab:random_sparse}
\resizebox{\textwidth}{!}{
\begin{tabular}{lcccccccccc}
\toprule
\textbf{Method} & \textbf{SST-2} & \textbf{RTE} & \textbf{BoolQ} & \textbf{WSC} & \textbf{WiC} & \textbf{MultiRC} & \textbf{SQuAD} & \textbf{ReCoRD} & \textbf{AVG.} \\
\midrule
Random Sparse & 92.35 & 64.62 & 73.57 & 61.33 & 58.97 & 61.43 & 87.88 & 80.73 & 72.22 \\
\method{} & \textbf{94.00} & \textbf{66.19} & \textbf{76.30} & \textbf{62.51} & \textbf{60.20} & \textbf{70.43} & \textbf{88.49} & \textbf{81.40} & \textbf{74.34} \\
\midrule
\textit{Improvement} & \textit{+1.65} & \textit{+1.57} & \textit{+2.73} & \textit{+1.18} & \textit{+1.23} & \textit{+9.00} & \textit{+0.61} & \textit{+0.67} & \textit{\textbf{+2.12}} \\
\bottomrule
\end{tabular}
}
\end{table*}

As shown in \Cref{tab:random_sparse}, \method{} consistently outperforms uniform random selection by 2.12\% on average. The performance gap is particularly pronounced on complex reasoning and comprehension tasks, with significant gains of up to 9.00\% on MultiRC and 2.73\% on BoolQ. 

This rigorously confirms that the MAB policy is strictly necessary. While raw sparsity alone provides the computational acceleration, uniformly discarding layers leads to a severe loss of critical gradient signals. \method{} effectively identifies and focuses updates on the most sensitive parameters, preventing degradation and maintaining robust accuracy under high sparsity constraints.}

\begin{figure}[htbp]
    \centering
    \includegraphics[width=\columnwidth]{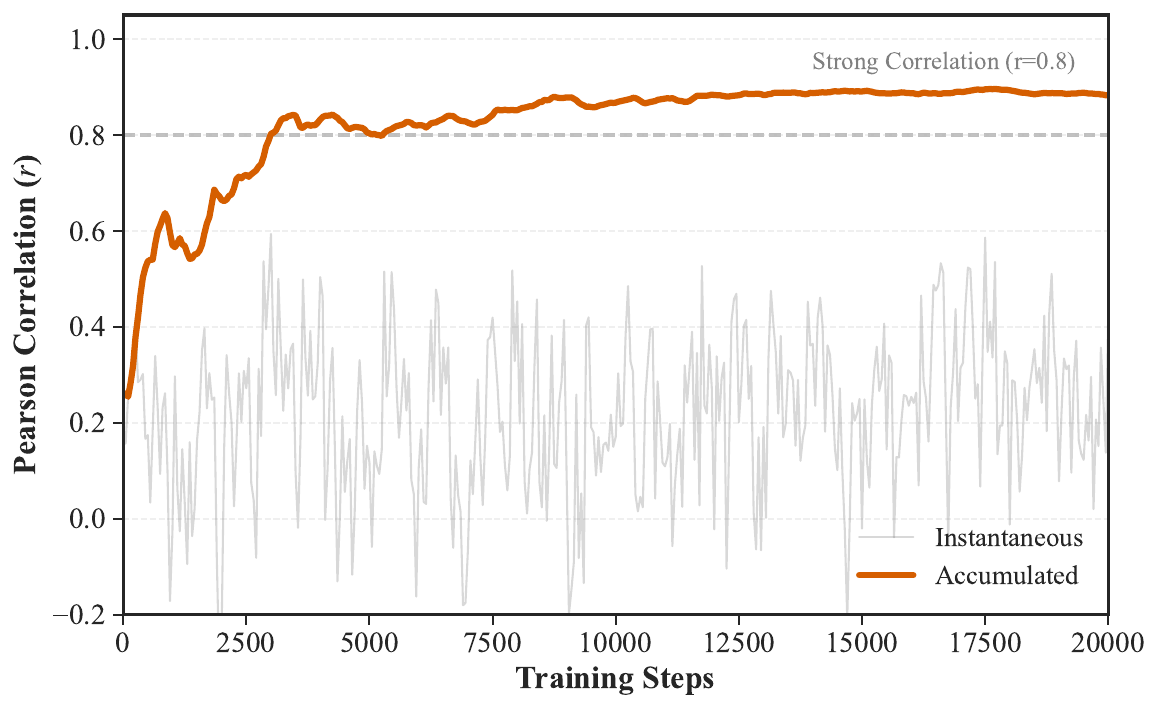}
    \caption{
    The Pearson correlation between the layer sampling probabilities assigned by \method{} and the gradient norms computed by Adam. Instantaneous correlation exhibits substantial fluctuations, reflecting the high variance inherent in  ZO estimates. In contrast, accumulated statistics converge steadily to $r \approx 0.88$, demonstrating that \method{} effectively recovers true layer sensitivity through temporal aggregation.
    }
    \label{fig:learning_curve}
\end{figure} 
\begin{figure}[htbp]
    \centering
    \includegraphics[width=\columnwidth]{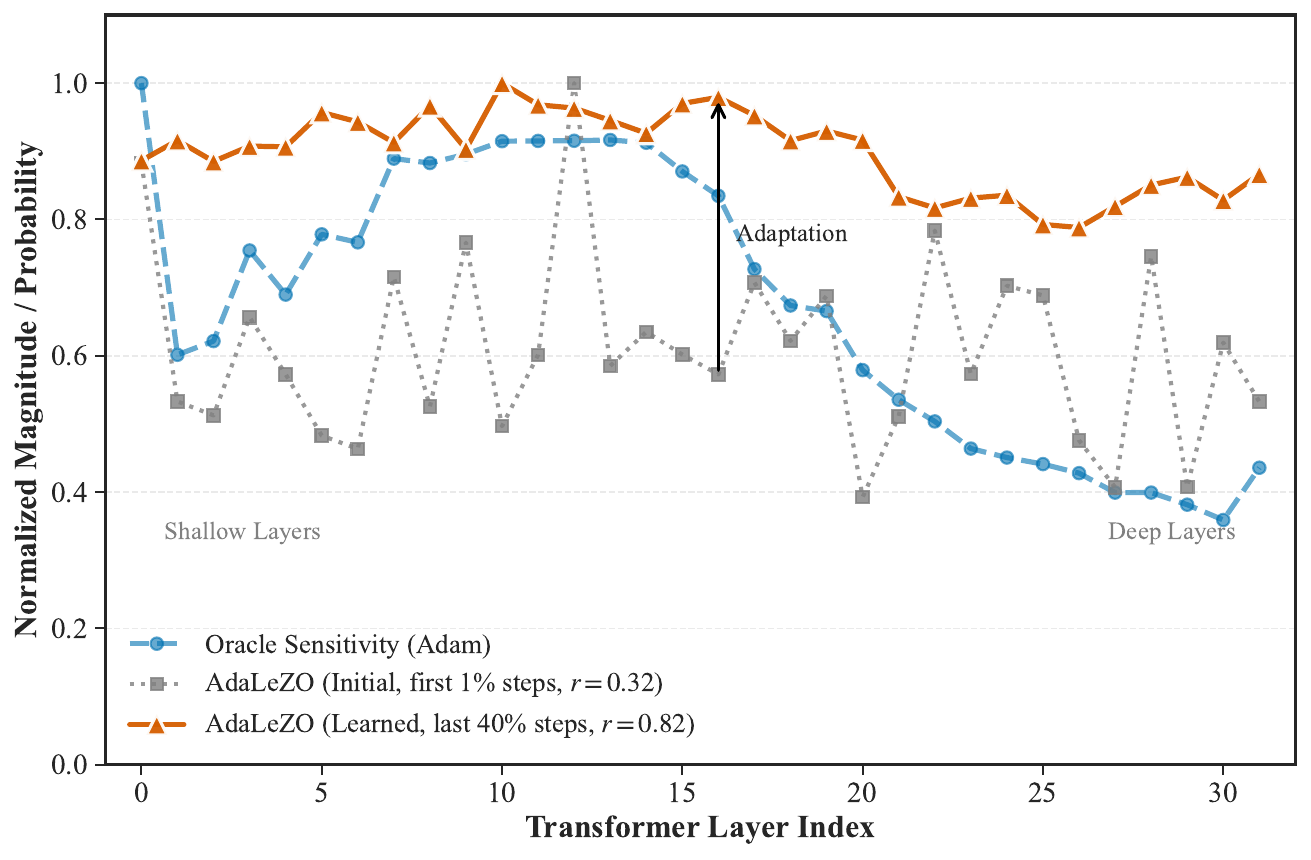}
    \caption{
    Layer-wise sensitivity alignment on OPT-6.7b (SST-2). We compare the ground truth sensitivity profile (derived from Adam's accumulated gradient norms) with \method{}'s learned sampling probabilities. While the initial policy (Gray, $r=0.32$) is nearly random, \method{} autonomously converges to a policy (Orange, $r=0.82$) that strongly correlates with the oracle sensitivity, demonstrating its ability to identify important layers without first-order gradients.
    }
    \label{fig:evolution_proof}
\end{figure}
\section{Mechanism Analysis}
\label{subsec:analysis}

\paragraph{Analyzing the Structural Learning Capability of \method{}.} 
A central question arises: \textit{Given only noisy Zeroth-Order (ZO) scalar feedback, does \method{} genuinely ``learn'' the hierarchical structure of the model, or is it merely performing a random walk?}

To address this, we tracked the Pearson correlation coefficient $r$ between the layer sampling probability distribution $\pi_t$ of \method{} and the Oracle (gradient norms calculated by Adam). As shown in \Cref{fig:learning_curve}, we observed a significant \textbf{``Temporal Denoising''} phenomenon:

\textbf{Stochasticity of Instantaneous Estimates:} The gray curve illustrates that the single-step instantaneous correlation exhibits severe oscillations ($r \in [-0.1, 0.6]$). This observation aligns with the theoretical properties of zeroth-order optimization. Specifically, since probing occurs along a single random direction $z$ at each step, the individual gradient estimate $\hat{g}$ suffers from extremely high variance, causing the algorithm to appear engaged in disordered exploration over short time scales.
    
\textbf{Robustness of Accumulated Signals:} Conversely, examining the cumulative mean of the sampling probabilities, denoted as $\bar{\pi}_T = \frac{1}{T}\sum_{t=1}^T \pi_t$ (solid orange line), reveals a distinct trend. The correlation steadily ascends from an initial value of $0.3$, ultimately reaching $0.88$ upon convergence.

Furthermore, in \Cref{fig:evolution_proof}, we visualize the cumulative probability distribution curves of the \method{} layer sampling at both the initial and final stages of training. It is clearly observable that during the early stage, the activation probability distribution across layers exhibits high variance, resembling random noise, and diverges significantly from the true gradient distribution. In contrast, during the late stage, the disparities in activation probabilities across layers diminish, resulting in a smoother curve. Concurrently, the correlation with the true gradient increases from $r=0.32$ initially to $r=0.82$. This demonstrates that \method{} effectively achieves denoising and approximates the true gradient distribution.

\textbf{Conclusion:} These results provide compelling evidence that \method{} functions as a \textbf{Temporal Filter}. While individual time steps are dominated by stochastic noise, the true gradient signal maintains consistency across the temporal dimension. Through continuous updates via the Bandit mechanism, \method{} successfully integrates low-frequency structural signals from high-frequency noise, thereby accurately reconstructing the model's intrinsic parameter sensitivity distribution without computing backpropagation.

\begin{figure*}[t!]
    \centering
    \begin{subfigure}[b]{0.24\textwidth}
        \centering
        \includegraphics[width=\textwidth]{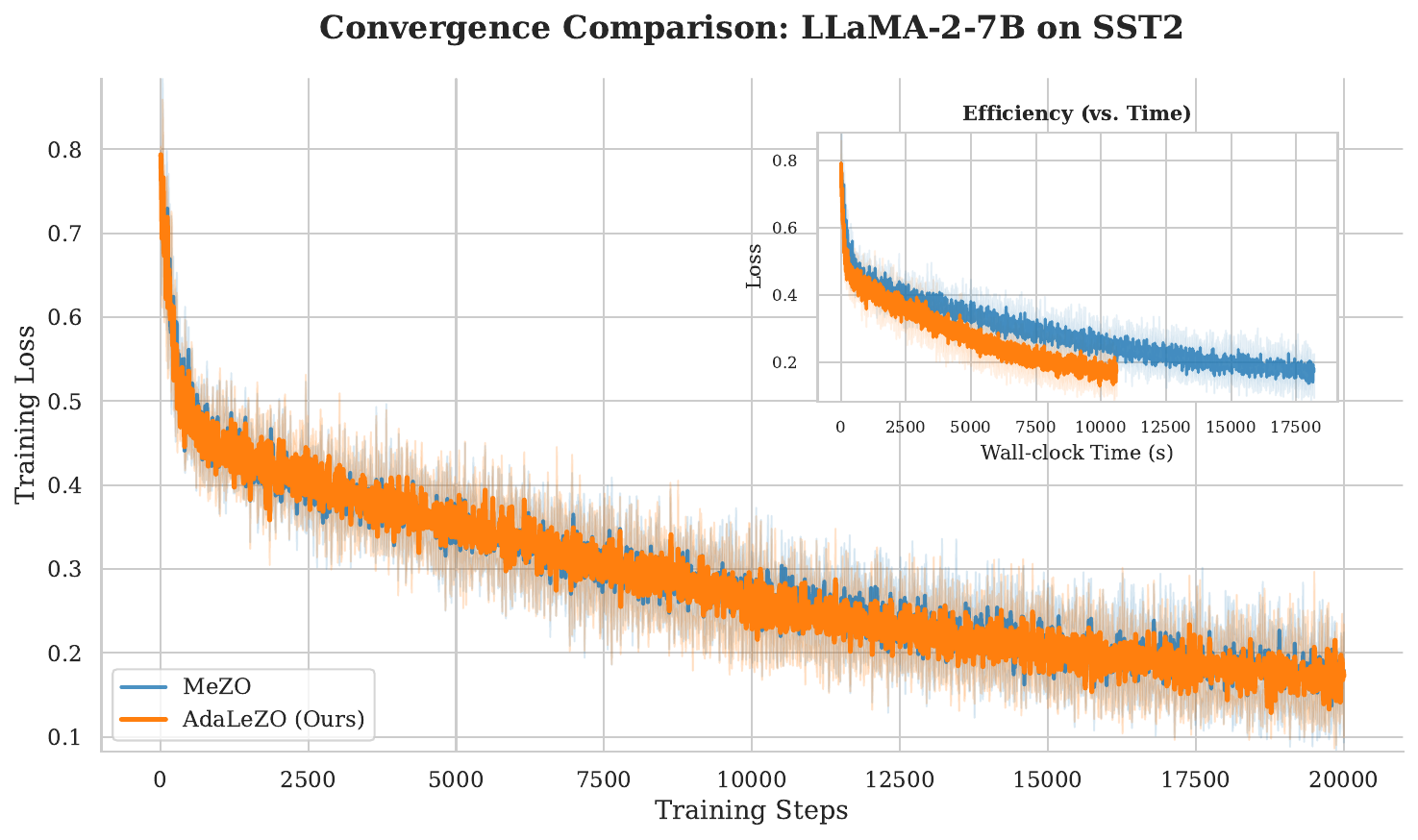}
        \caption{LLaMA-2-7B \& SST2}
        \label{fig:llama2_sst2}
    \end{subfigure}
    \hfill
    \begin{subfigure}[b]{0.24\textwidth}
        \centering
        \includegraphics[width=\textwidth]{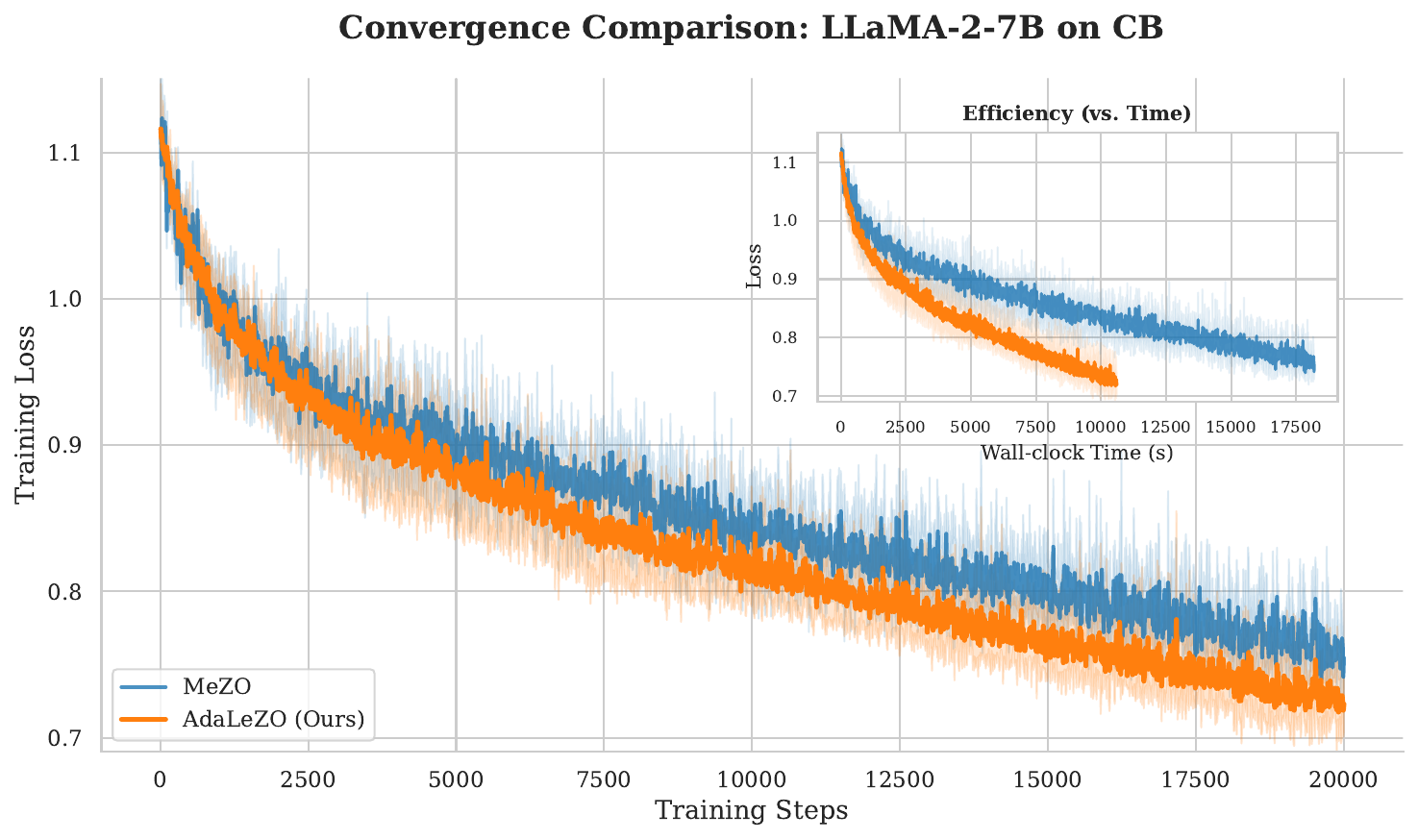}
        \caption{LLaMA-2-7B \& CB}
        \label{fig:llama2_cb}
    \end{subfigure}
    \hfill
    \begin{subfigure}[b]{0.24\textwidth}
        \centering
        \includegraphics[width=\textwidth]{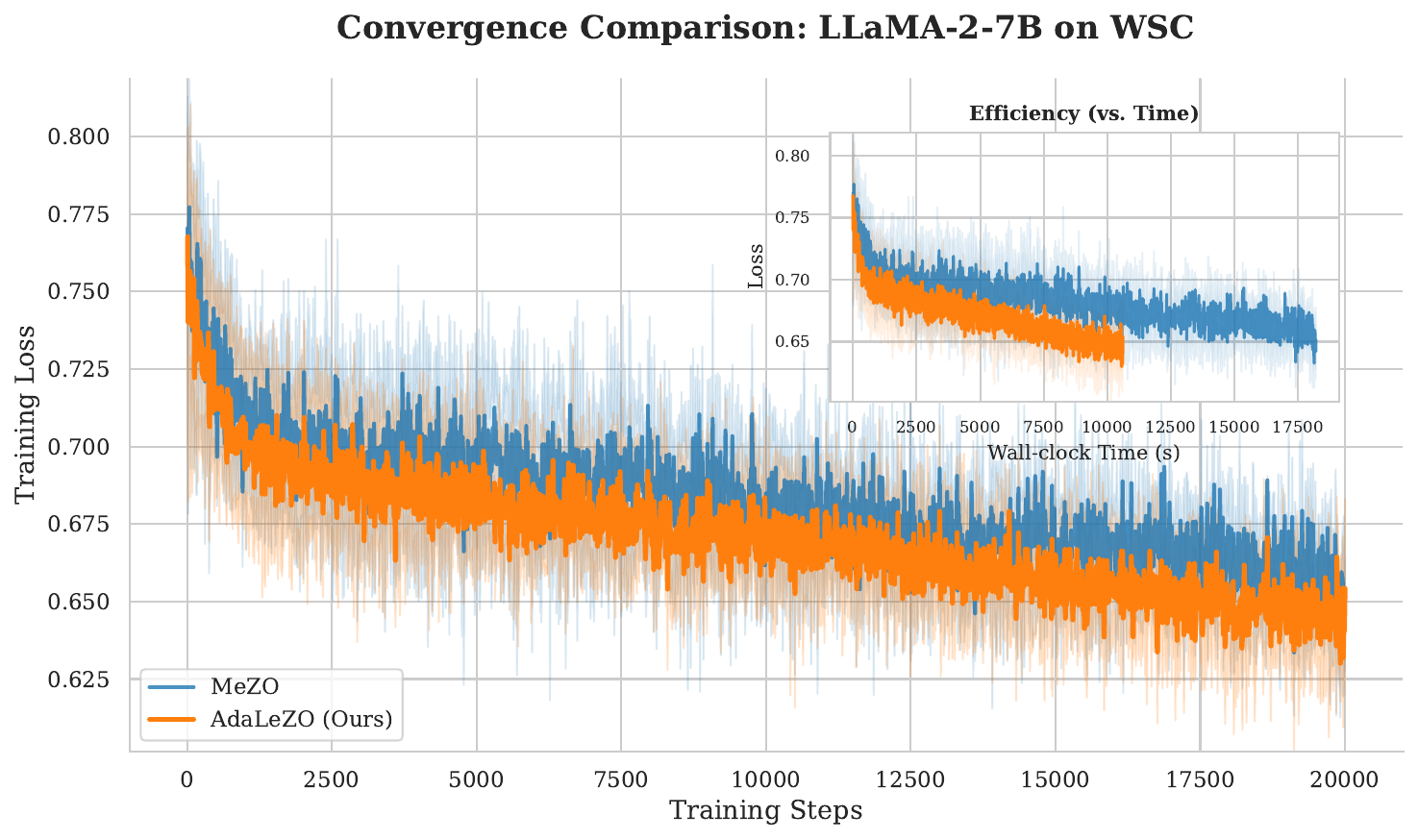}
        \caption{LLaMA-2-7B \& WSC}
        \label{fig:llama2_wsc}
    \end{subfigure}
    \hfill
    \begin{subfigure}[b]{0.24\textwidth}
        \centering
        \includegraphics[width=\textwidth]{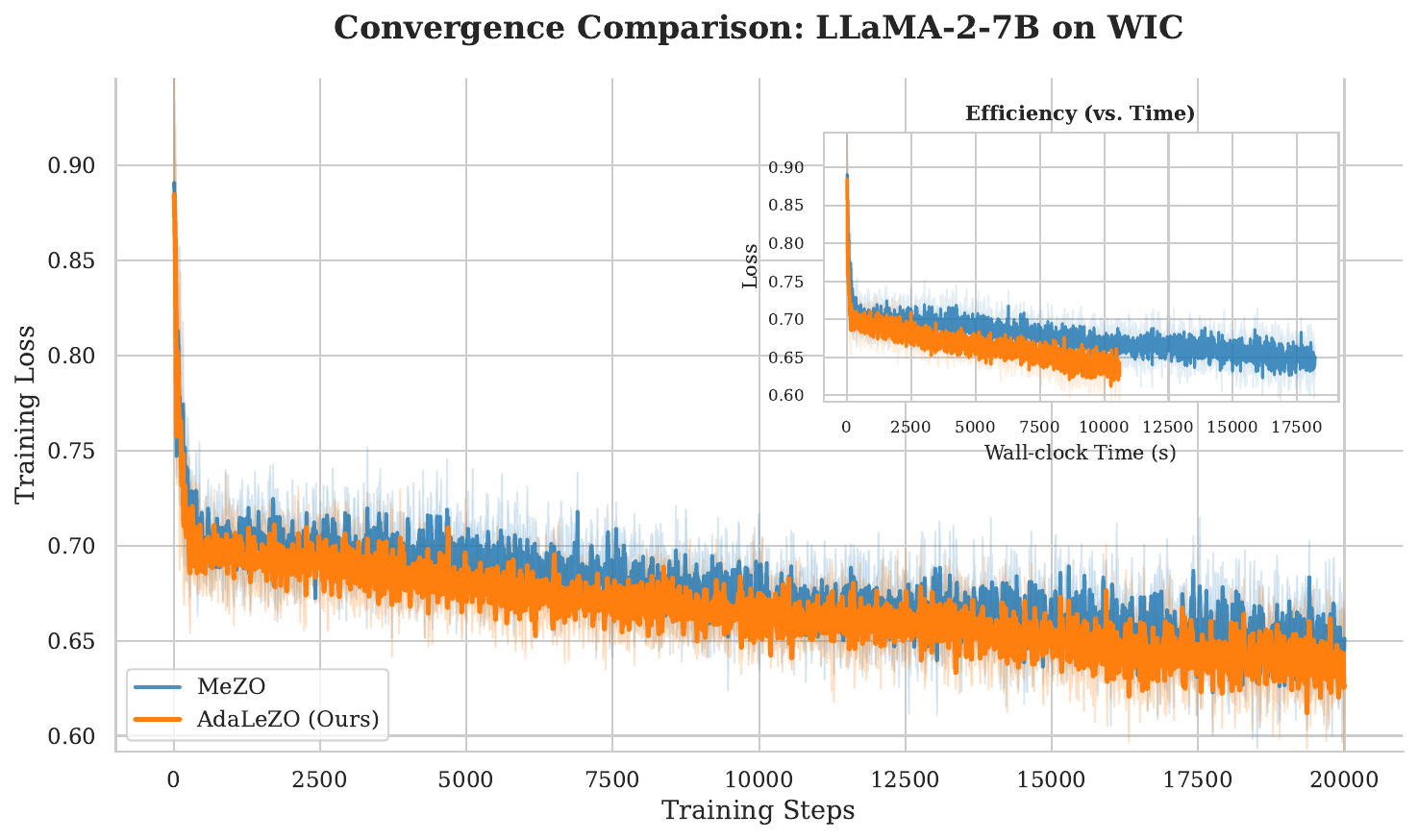}
        \caption{LLaMA-2-7B \& WIC}
        \label{fig:llama2_wic}
    \end{subfigure}

    \begin{subfigure}[b]{0.24\textwidth}
        \centering
        \includegraphics[width=\textwidth]{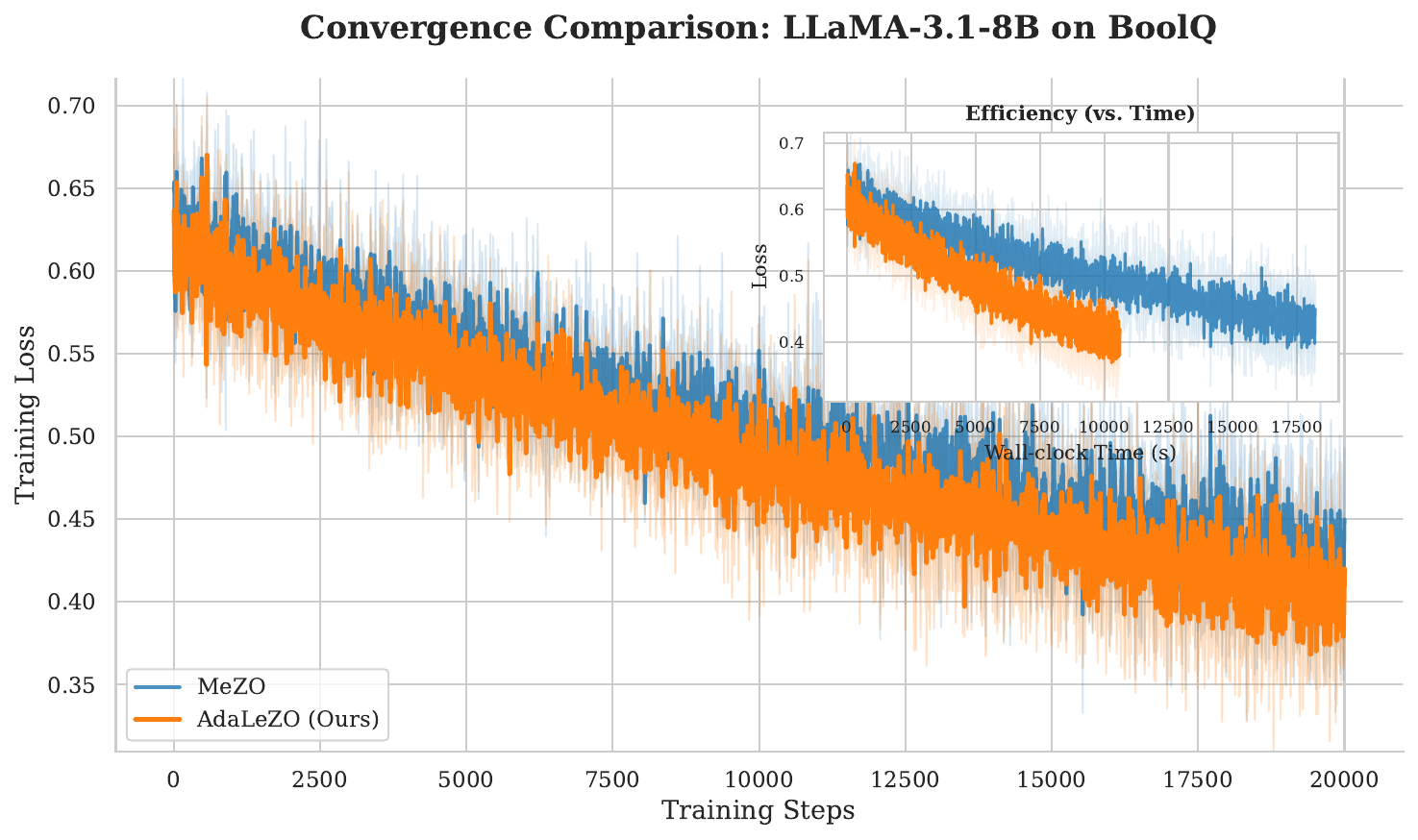}
        \caption{LLaMA-3.1-8B \& BoolQ}
        \label{fig:llama3_1_sst2}
    \end{subfigure}
    \hfill
    \begin{subfigure}[b]{0.24\textwidth}
        \centering
        \includegraphics[width=\textwidth]{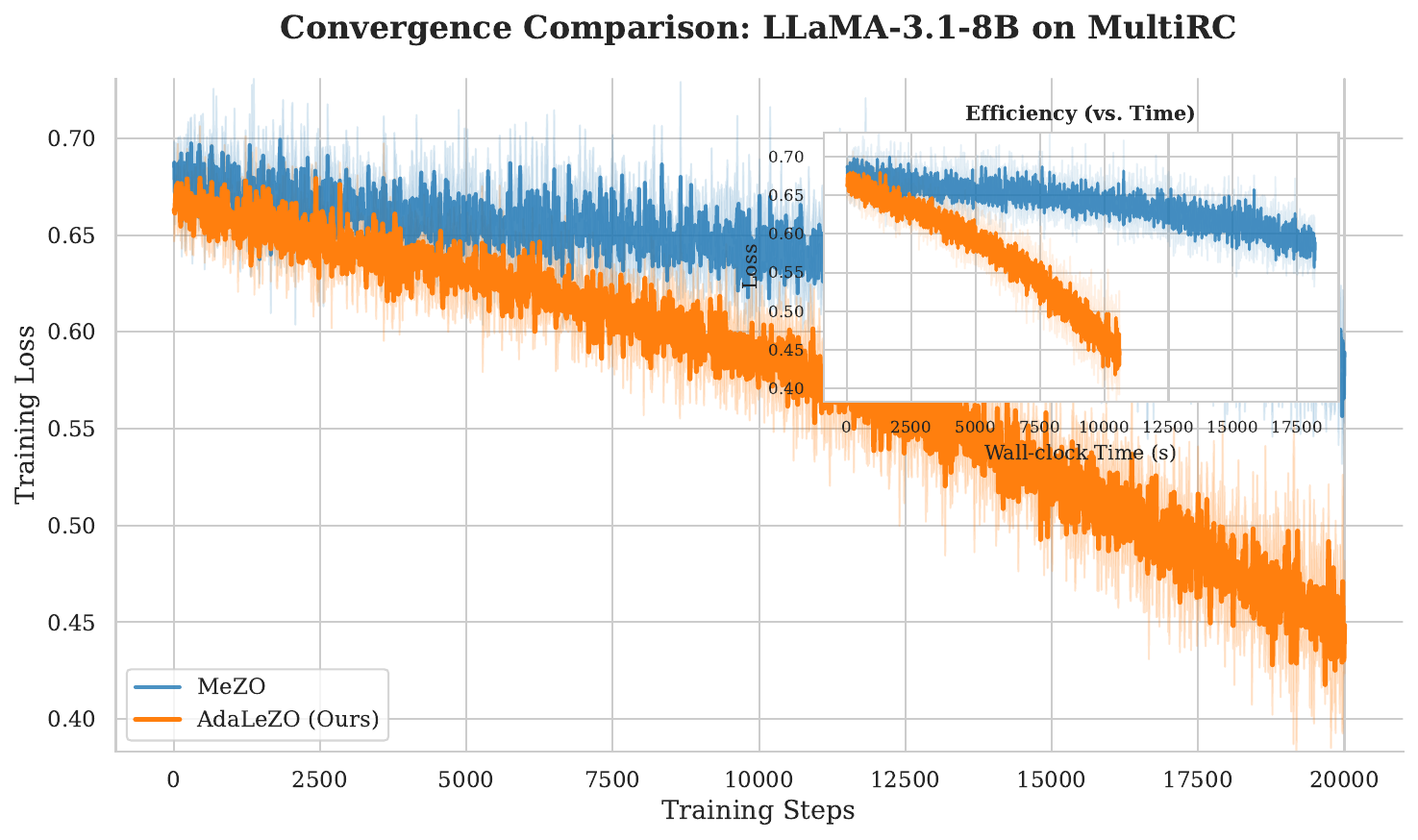}
        \caption{LLaMA-3.1-8B \& MultiRC}
        \label{fig:llama3_1_cb}
    \end{subfigure}
    \hfill
    \begin{subfigure}[b]{0.24\textwidth}
        \centering
        \includegraphics[width=\textwidth]{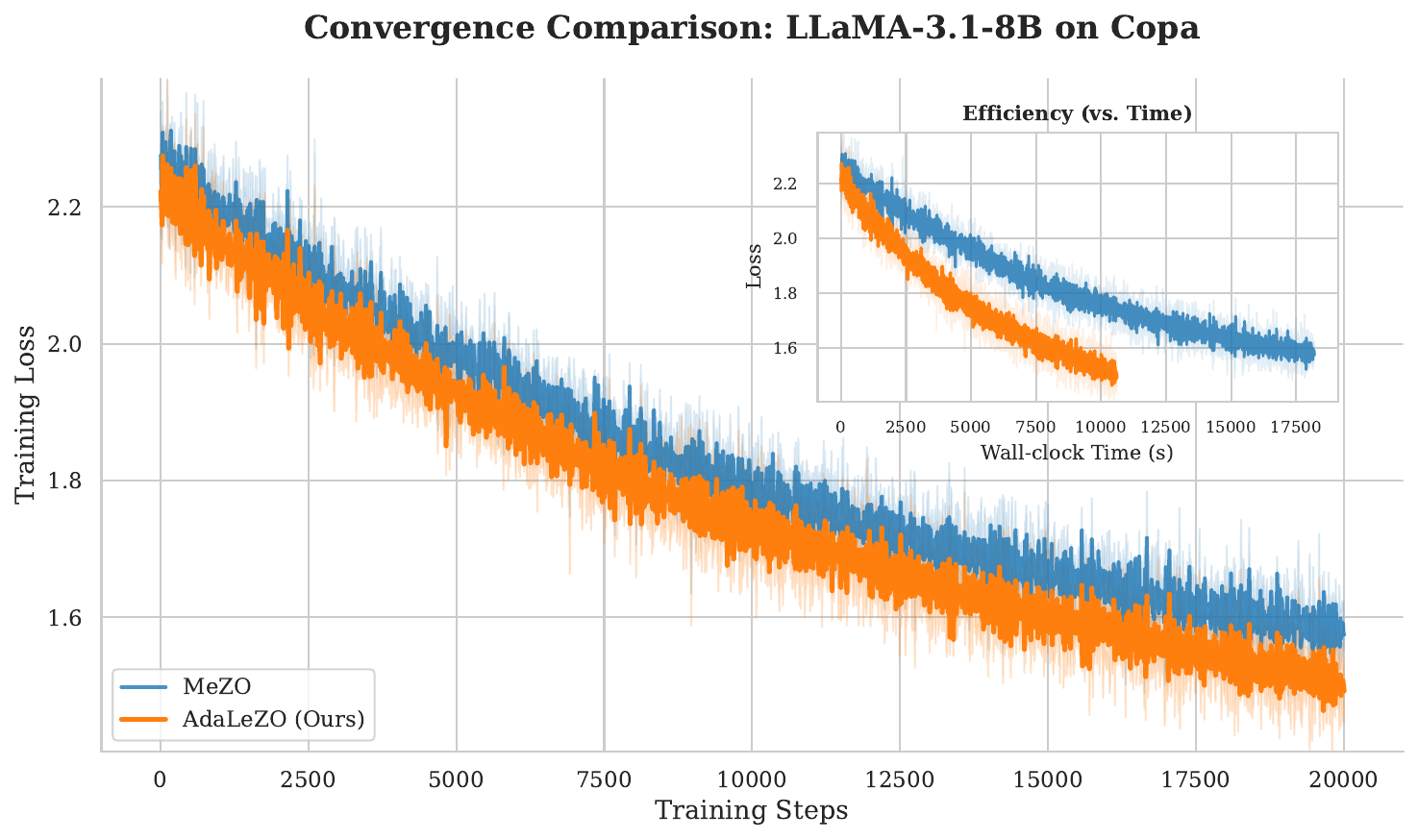}
        \caption{LLaMA-3.1-8B \& Copa}
        \label{fig:llama3_1_wsc}
    \end{subfigure}
    \hfill
    \begin{subfigure}[b]{0.24\textwidth}
        \centering
        \includegraphics[width=\textwidth]{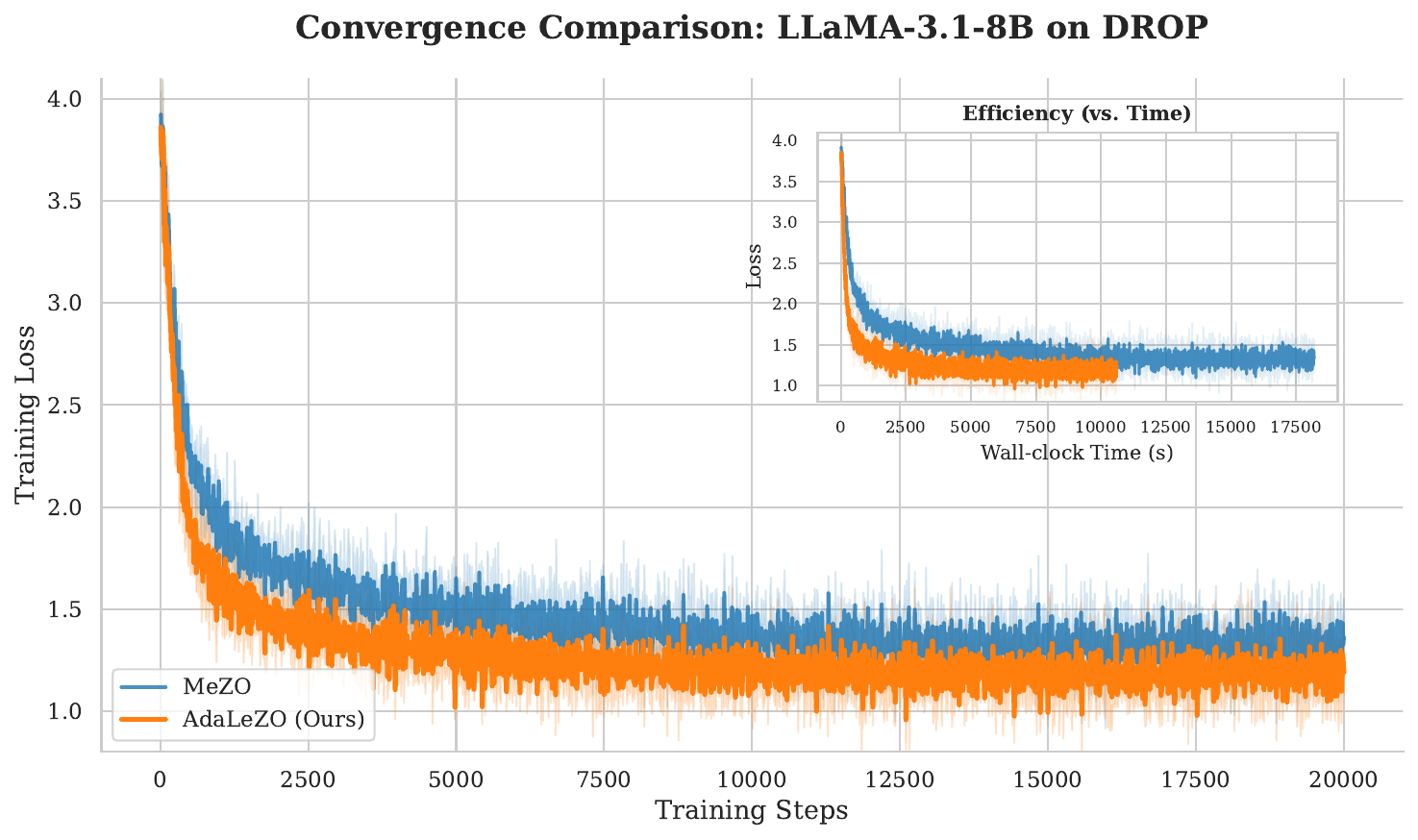}
        \caption{LLaMA-3.1-8B \& DROP}
        \label{fig:llama3_1_wic}
    \end{subfigure}
    \caption{
    Loss convergence curves for fine-tuning LLaMA models using ZO optimizers. In the main plot, the x-axis represents training steps, while in the inset, it indicates wall-clock time. Evidently, under the same number of fine-tuning steps, \method{} achieves faster convergence compared to MeZO.
    }
    \label{fig:llama_loss_convergence}
\end{figure*}
\section{Convergence Analysis}
\textbf{\method{} Achieves Faster Convergence than MeZO.}
\Cref{fig:llama_loss_convergence} illustrates the loss convergence rates of fine-tuning LLaMA models across various tasks using both \method{} and MeZO. Under identical learning rates, \method{} consistently demonstrates faster convergence than MeZO, even without considering wall-clock time acceleration. This advantage is particularly pronounced on the more robust LLaMA-3.1-8B model; for example, on the MultiRC task, \method{} outpaces MeZO by more than $2\times$ in terms of convergence speed. The improvements in wall-clock time are even more substantial across all tasks. These results highlight the effectiveness of \method{} in variance reduction, enabling the model to rapidly converge to regions of lower loss.

\section{Theoretical Analysis}
\label{app:theory}

In this section, we provide a comprehensive convergence analysis of the \method{} algorithm. We rely on the Gaussian smoothing framework to analyze the properties of ZO optimization. We first introduce the necessary notations and assumptions, then analyze the bias and variance of the \method{} estimator, and finally establish the convergence rate for non-convex objective functions.

\subsection{Preliminaries and Assumptions}

Consider the optimization problem $\min_{\theta \in \mathbb{R}^d} \mathcal{L}(\theta)$. We partition the parameter vector $\theta$ into $L$ groups according to layer structure, denoted as $\{\theta^{(1)}, \dots, \theta^{(L)}\}$.

\paragraph{Gaussian Smoothing.}
Following Nesterov and Spokoiny \citep{nesterov2017random}, we define the Gaussian smoothed approximation of $\mathcal{L}$ with smoothing parameter $\mu > 0$ as:
\begin{equation}
    \mathcal{L}_\mu(\theta) = \mathbb{E}_{u \sim \mathcal{N}(0, I_d)} [\mathcal{L}(\theta + \mu u)].
\end{equation}
It is well-known that $\mathcal{L}_\mu$ is continuously differentiable even if $\mathcal{L}$ is non-smooth. The gradient of $\mathcal{L}_\mu$ is given by:
\begin{equation}
    \nabla \mathcal{L}_\mu(\theta) = \mathbb{E}_{u \sim \mathcal{N}(0, I_d)} \left[ \frac{\mathcal{L}(\theta + \mu u) - \mathcal{L}(\theta)}{\mu} u \right].
\end{equation}

\paragraph{Assumptions.}
We make the following standard assumptions for non-convex ZO optimization~\citep{ghadimi2013stochastic, liu2018zeroth}.

\begin{assumption}[L-Smoothness]
\label{ass:smoothness}
The objective function $\mathcal{L}(\theta)$ is differentiable and $L_g$-smooth, meaning its gradient is Lipschitz continuous:
\begin{equation}
\begin{split}
    \| \nabla \mathcal{L}(\theta_1) - \nabla \mathcal{L}(\theta_2) \| \le L_g \| \theta_1 - \theta_2 \|, \\
    \quad \forall \theta_1, \theta_2 \in \mathbb{R}^d.
\end{split}
\end{equation}
\end{assumption}

\begin{assumption}[Bounded Gradient]
\label{ass:bounded_grad}
The gradient of the objective function is bounded: 
\begin{equation}
    \mathbb{E}[\|\nabla \mathcal{L}(\theta)\|^2] \le G^2.
\end{equation}
 
\end{assumption}
\Cref{ass:bounded_grad} is often relaxed in proofs, but it helps simplify the variance bound.

\begin{assumption}[Function Boundedness]
\label{ass:bounded_func}
The optimal function value is bounded from below:
\begin{equation}
    \mathcal{L}^* = \inf_{\theta} \mathcal{L}(\theta) > -\infty.
\end{equation}
\end{assumption}

\paragraph{Key Lemmas for Gaussian Smoothing.}
Based on \Cref{ass:smoothness}, the smoothed function $\mathcal{L}_\mu$ inherits smoothness and approximates $\mathcal{L}$ well.

\begin{lemma}[Properties of $\mathcal{L}_\mu$, \citep{nesterov2017random}]
\label{lem:smoothing_props}
Under \Cref{ass:smoothness}, for any $\theta \in \mathbb{R}^d$:
\begin{enumerate}
    \item The gradient difference is bounded: $\| \nabla \mathcal{L}_\mu(\theta) - \nabla \mathcal{L}(\theta) \| \le \frac{\mu d L_g}{2}$.
    \item $\mathcal{L}_\mu$ is $L_\mu$-smooth with $L_\mu \le L_g$.
    \item $\mathbb{E}_{u \sim \mathcal{N}(0, I)} [\| \nabla \mathcal{L}_\mu(\theta) - \frac{\mathcal{L}(\theta + \mu u) - \mathcal{L}(\theta)}{\mu} u \|^2] \le \mu^2 d^2 L_g^2 + d \|\nabla \mathcal{L}(\theta)\|^2$.
\end{enumerate}
\end{lemma}

\subsection{Estimator Analysis}

\fei{\paragraph{Subspace Smoothing vs. Full Smoothing.}
Unlike standard ZO methods that perturb the entire parameter space, \method{} generates perturbations strictly within an active subspace determined by the sampled layers $\mathcal{I}_t$. From the perspective of Randomized Block Coordinate Descent (RBCD), the finite difference on this sparse subspace estimates the gradient of the objective, smoothed only over that specific subspace. While masking introduces cross-layer Hessian interference relative to full-space Gaussian smoothing, Taylor expansion shows that these cross-layer terms have zero expectation to first order due to the independence of layer-wise perturbations. As $\mu \rightarrow 0$, both the subspace-smoothed and full-smoothed expectations asymptotically converge to the true non-smooth gradient $\nabla \mathcal{L}(\theta)$, with the discrepancy rigorously bounded by $O(\mu d L_g)$. The primary structural bias in our practical implementation arises not from this sparse sampling itself, but from the variance-control clipping mechanism, which we explicitly bound in the following sections.}

Let $z_t \sim \mathcal{N}(0, I_d)$ be the random perturbation vector at step $t$. The standard symmetric difference ZO estimator is:
\begin{equation}
    \hat{g}_t^{\text{ZO}} = \frac{\mathcal{L}(\theta_t + \mu z_t) - \mathcal{L}(\theta_t - \mu z_t)}{2\mu} z_t.
\end{equation}
From \citep{malladi2023fine}, we know $\mathbb{E}_{z_t}[\hat{g}_t^{\text{ZO}}] = \nabla \mathcal{L}_\mu(\theta_t)$.

\method{} applies a masking matrix $M_t$ based on sampling. Let $p_t \in \Delta^{L-1}$ be the sampling distribution. 
Let $\rho \in (0, 1]$ be the sampling ratio parameter (corresponding to \texttt{adalezo\_k\_ratio}). 
The number of sampling draws is determined by $K = \max(1, \lfloor \rho L \rfloor)$.
We perform $K$ independent draws with replacement. Let $N_{t,l} \in \{0, \dots, K\}$ be the number of times layer $l$ is selected. The \method{} estimator is:
\begin{equation}
    \hat{g}_t^{\text{Ada},(l)} = \frac{N_{t,l}}{K p_t(l)} \hat{g}_t^{\text{ZO},(l)}.
\end{equation}

\subsubsection{Unbiasedness}
\begin{theorem}[Unbiasedness wrt Smoothed Gradient]
\label{thm:unbiasedness}
The \method{} estimator is an unbiased estimator of the smoothed gradient $\nabla \mathcal{L}_\mu(\theta_t)$.
\end{theorem}
\begin{proof}
First, we analyze the expectation conditioned on the perturbation $z_t$. The randomness comes solely from the sampling counts $N_{t,l}$. Since sampling is with replacement, $N_{t,l} \sim \text{Binomial}(K, p_t(l))$, implies $\mathbb{E}[N_{t,l}] = K p_t(l)$.
\begin{equation}
\begin{split}
    \mathbb{E}_{\mathcal{S}_t} [\hat{g}_t^{\text{Ada},(l)} \mid z_t] &= \mathbb{E}_{\mathcal{S}_t} \left[ \frac{N_{t,l}}{K p_t(l)} \hat{g}_t^{\text{ZO},(l)} \right] \\
    &= \frac{\hat{g}_t^{\text{ZO},(l)}}{K p_t(l)} \mathbb{E}_{\mathcal{S}_t}[N_{t,l}] \\
    &= \hat{g}_t^{\text{ZO},(l)}.
\end{split}
\end{equation}
Since this holds for all layers, $\mathbb{E}_{\mathcal{S}_t} [\hat{g}_t^{\text{Ada}} \mid z_t] = \hat{g}_t^{\text{ZO}}$.
Taking the expectation over $z_t$:
\begin{equation}
    \mathbb{E}_{z_t, \mathcal{S}_t} [\hat{g}_t^{\text{Ada}}] = \mathbb{E}_{z_t} [\hat{g}_t^{\text{ZO}}] = \nabla \mathcal{L}_\mu(\theta_t).
\end{equation}

\begin{proposition}[Finite-C Bias Bound]
\label{prop:bias_bound}
Under \Cref{ass:bounded_grad}, the bias introduced by the clipping threshold $C$ is bounded by:
\begin{equation}
\begin{split}
    \| \text{Bias} \| & = \| \mathbb{E}[\tilde{g}_t^{\text{Ada}}] - \nabla \mathcal{L}_\mu(\theta_t) \| \\
    & \le G \sum_{l: p_t(l) < \frac{1}{CK}} (1 - C K p_t(l)).
\end{split}
\end{equation}
\end{proposition}
\begin{proof}
The expected value of the clipped estimator for layer $l$ is $\mathbb{E}_{\mathcal{S}_t}[\tilde{g}_t^{\text{Ada},(l)}] = \tilde{W}_{t,l} \mathbb{E}[N_{t,l}] \hat{g}_t^{\text{ZO},(l)} = \min(1, C K p_t(l)) \hat{g}_t^{\text{ZO},(l)}$.
The bias is the norm of the difference between this expectation and the unbiased ZO gradient. Notice that $(W_{t,l} - \tilde{W}_{t,l}) > 0$ if and only if $C K p_t(l) < 1$, i.e., $p_t(l) < \frac{1}{CK}$.
Applying the triangle inequality and taking the expectation over $z_t$:
\begin{equation}
    \| \text{Bias} \| \le \sum_{l: p_t(l) < \frac{1}{CK}} (1 - C K p_t(l)) \mathbb{E} [\| \hat{g}_t^{\text{ZO},(l)} \|].
\end{equation}
Since $\mathbb{E}[\|\hat{g}_t^{\text{ZO},(l)}\|] \le \mathbb{E}[\|\hat{g}_t^{\text{ZO}}\|] \le G$ (derived from \Cref{ass:bounded_grad} and \Cref{lem:smoothing_props}), the bound holds.
\end{proof}
This explicit formulation corroborates our ablation study: setting a moderate $C$ (e.g., $C=4$) ensures that only layers with negligibly small sampling probabilities contribute to the bounded bias, effectively balancing the bias-variance trade-off.
\end{proof}

\subsubsection{Variance Reduction}
\begin{theorem}[Variance Decomposition and Optimality]
\label{thm:variance}
Let $v_t^{(l)} = (\hat{g}_t^{\text{ZO},(l)})^2$ be the squared magnitude of the dense ZO gradient on layer $l$. Conditioned on $z_t$, the total variance of the \method{} estimator is:
\begin{equation}
    \text{Var}_{\mathcal{S}_t}(\hat{g}_t^{\text{Ada}} \mid z_t) = \frac{1}{K} \sum_{l=1}^L v_t^{(l)} \left( \frac{1}{p_t(l)} - 1 \right).
\end{equation}
This variance is minimized when $p_t(l) \propto \sqrt{v_t^{(l)}} = |\hat{g}_t^{\text{ZO},(l)}|$.
\end{theorem}

\begin{proof}
Using the properties of the Binomial distribution $\text{Var}(N_{t,l}) = K p_t(l)(1 - p_t(l))$:
\begin{equation}
\begin{split}
    \text{Var}(\hat{g}_t^{\text{Ada},(l)} \mid z_t) &= \left( \frac{\hat{g}_t^{\text{ZO},(l)}}{K p_t(l)} \right)^2 \text{Var}(N_{t,l}) \\
    &= \frac{v_t^{(l)}}{K^2 p_t(l)^2} \cdot K p_t(l)(1 - p_t(l)) \\
    &= \frac{v_t^{(l)}}{K} \left( \frac{1}{p_t(l)} - 1 \right).
\end{split}
\end{equation}
Summing over layers (assuming independence of sampling counts between layers implies additivity of variance for the norm, or simply analyzing the trace of the covariance):
\begin{equation}
\begin{split}
    V(p) & = \sum_{l=1}^L \text{Var}(\hat{g}_t^{\text{Ada},(l)} \mid z_t) \\
    & = \frac{1}{K} \left( \sum_{l=1}^L \frac{v_t^{(l)}}{p_t(l)} - \sum_{l=1}^L v_t^{(l)} \right).
\end{split}
\end{equation}
To minimize $V(p)$ subject to $\sum p_t(l) = 1$, we use the Cauchy-Schwarz inequality on the first term:
\begin{equation}
    \left( \sum_{l=1}^L \frac{v_t^{(l)}}{p_t(l)} \right) \left( \sum_{l=1}^L p_t(l) \right) \ge \left( \sum_{l=1}^L \sqrt{v_t^{(l)}} \right)^2.
\end{equation}
Equality holds when $\sqrt{v_t^{(l)}}/p_t(l) = c$, i.e., $p_t(l) \propto \sqrt{v_t^{(l)}} = |\hat{g}_t^{\text{ZO},(l)}|$.
Thus, allocating the sampling probability proportional to the gradient magnitude minimizes the estimation variance.
\end{proof}

\fei{While \Cref{thm:variance} demonstrates the theoretical optimality of adaptive sampling, the unclipped IPW weight can cause infinite variance if $p_t(l) \rightarrow 0$. The clipping mechanism resolves this while explicitly preserving the fundamental dimension dependence.

\begin{theorem}[Explicit Variance Bound with Clipping]
\label{thm:explicit_variance}
For the clipped \method{} estimator $\tilde{g}_t^{\text{Ada}}$, the second moment is strictly capped by the clipping threshold $C$, while preserving the intrinsic dimension dependence $\mathcal{O}(d)$ of the ZO gradient:
\begin{equation}
    \mathbb{E} [\| \tilde{g}_t^{\text{Ada}} \|^2] \le (C + 1) \mathbb{E} [\| \hat{g}_t^{\text{ZO}} \|^2].
\end{equation}
\end{theorem}
\begin{proof}
For a specific layer $l$, the second moment is $\mathbb{E}[\|\tilde{g}_t^{\text{Ada},(l)}\|^2] = \tilde{W}_{t,l}^2 \mathbb{E}[N_{t,l}^2] \|\hat{g}_t^{\text{ZO},(l)}\|^2$.
For the Binomial variable $N_{t,l}$, $\mathbb{E}[N_{t,l}^2] = K p_t(l)(1 - p_t(l)) + K^2 p_t(l)^2 \le K p_t(l) + K^2 p_t(l)^2$. We bound the multiplier $\mathcal{M}_l = \tilde{W}_{t,l}^2 \mathbb{E}[N_{t,l}^2]$ by analyzing two disjoint cases:

\textbf{Case 1: No clipping occurs} ($\frac{1}{K p_t(l)} \le C$, which implies $K p_t(l) \ge \frac{1}{C}$).
\begin{equation}
\begin{split}
    \mathcal{M}_l & \le \frac{1}{K^2 p_t(l)^2} (K p_t(l) + K^2 p_t(l)^2) \\
    & = \frac{1}{K p_t(l)} + 1 \le C + 1.
\end{split}
\end{equation}

\textbf{Case 2: Clipping occurs} ($\frac{1}{K p_t(l)} > C$, which implies $K p_t(l) < \frac{1}{C}$).
\begin{equation}
\begin{split}
    \mathcal{M}_l & = C^2 (K p_t(l) + K^2 p_t(l)^2) \\ 
    & \le C^2 \left(\frac{1}{C} + \frac{1}{C^2}\right) = C + 1.
\end{split}
\end{equation}

In all scenarios, the multiplier $\mathcal{M}_l \le C + 1$. Summing the contributions over all layers:
\begin{equation}
\begin{split}
    \mathbb{E} [\| \tilde{g}_t^{\text{Ada}} \|^2] & \le (C + 1) \sum_{l=1}^L \mathbb{E} [\| \hat{g}_t^{\text{ZO},(l)} \|^2] \\ 
    & = (C + 1) \mathbb{E} [\| \hat{g}_t^{\text{ZO}} \|^2].
\end{split}
\end{equation}
\end{proof}
This explicitly resolves any concerns regarding a ``dimension-free paradox''. The fundamental dimension dependence $\mathcal{O}(d)$ is correctly and fully preserved within the dense ZO gradient's second moment $\mathbb{E}[\|\hat{g}_t^{\text{ZO}}\|^2]$. The clipped IPW mechanism via MAB strictly bounds the variance amplification multiplier at $\mathcal{O}(C)$, preventing unresolvable divergence without artificially suppressing the intrinsic dimension scaling.}

\subsection{Convergence Rate Analysis}

We now prove the main convergence theorem.

\begin{theorem}[Non-Convex Convergence]
\label{thm:main_convergence}
Suppose \Cref{ass:smoothness,ass:bounded_func} hold. Let the learning rate be $\eta_t = \frac{1}{\sqrt{L_\mu T}}$. Then, \method{} yields the following bound on the gradient of the smoothed function:
\begin{equation}
\begin{split}
    \min_{0 \le t < T} \mathbb{E} [\|\nabla \mathcal{L}_\mu(\theta_t)\|^2] & \le \frac{2 \sqrt{L_\mu}(\mathcal{L}(\theta_0) - \mathcal{L}^*)}{\sqrt{T}} \\
    &\quad + \frac{C \sigma^2}{\sqrt{T}},    
\end{split}
\end{equation}
where $C$ is a constant related to the variance. Furthermore, considering the smoothing bias, for the original function $\mathcal{L}$:
\begin{equation}
    \mathbb{E} [\|\nabla \mathcal{L}(\theta_t)\|^2] \le O\left(\frac{1}{\sqrt{T}}\right) + O(\mu^2 d^2).
\end{equation}
\end{theorem}

\begin{proof}
\textbf{Step 1: Descent Lemma on Smoothed Function.}
Since $\mathcal{L}_\mu$ is $L_\mu$-smooth:
\begin{equation}
\begin{split}
    \mathcal{L}_\mu(\theta_{t+1}) \le \mathcal{L}_\mu(\theta_t) & + \langle \nabla \mathcal{L}_\mu(\theta_t), \theta_{t+1} - \theta_t \rangle \\
    & + \frac{L_\mu}{2} \| \theta_{t+1} - \theta_t \|^2.    
\end{split}
\end{equation}
Substituting the update rule $\theta_{t+1} = \theta_t - \eta_t \hat{g}_t^{\text{Ada}}$:
\begin{equation}
\begin{split}
    \mathcal{L}_\mu(\theta_{t+1}) \le \mathcal{L}_\mu(\theta_t) & - \eta_t \langle \nabla \mathcal{L}_\mu(\theta_t), \hat{g}_t^{\text{Ada}} \rangle \\
    & + \frac{L_\mu \eta_t^2}{2} \| \hat{g}_t^{\text{Ada}} \|^2.    
\end{split}
\end{equation}

\textbf{Step 2: Expectation over Sampling and Perturbation.}
Take the total expectation $\mathbb{E} = \mathbb{E}_{z_t, \mathcal{S}_t}[\cdot \mid \theta_t]$.
Using \Cref{thm:unbiasedness} ($\mathbb{E}[\hat{g}_t^{\text{Ada}}] = \nabla \mathcal{L}_\mu(\theta_t)$):
\begin{equation}
\begin{split}
    \mathbb{E}[\mathcal{L}_\mu(\theta_{t+1})] \le \mathcal{L}_\mu(\theta_t) & - \eta_t \| \nabla \mathcal{L}_\mu(\theta_t) \|^2 \\
    & + \frac{L_\mu \eta_t^2}{2} \mathbb{E}[\|\hat{g}_t^{\text{Ada}}\|^2].    
\end{split}
\end{equation}
Using the variance definition $\mathbb{E}[\|X\|^2] = \text{Var}(X) + \|\mathbb{E}[X]\|^2$:
\begin{equation}
    \mathbb{E}[\|\hat{g}_t^{\text{Ada}}\|^2] = \text{Var}(\hat{g}_t^{\text{Ada}}) + \|\nabla \mathcal{L}_\mu(\theta_t)\|^2.
\end{equation}
Let the variance be bounded by $\sigma^2$ (a combination of ZO variance and sampling variance).
\begin{equation}
\begin{split}
    \mathbb{E}[\mathcal{L}_\mu(\theta_{t+1})] & \le \mathcal{L}_\mu(\theta_t) \\
    & \quad - \eta_t \left(1 - \frac{L_\mu \eta_t}{2}\right) \|\nabla \mathcal{L}_\mu(\theta_t)\|^2 \\
    & \quad + \frac{L_\mu \eta_t^2 \sigma^2}{2}.    
\end{split}
\end{equation}

\textbf{Step 3: Telescoping Sum.}
Set $\eta_t = \frac{1}{\sqrt{T}}$. For large $T$, $1 - \frac{L_\mu \eta_t}{2} \ge \frac{1}{2}$. Rearranging:
\begin{equation}
    \frac{\eta_t}{2} \|\nabla \mathcal{L}_\mu(\theta_t)\|^2 \le \mathcal{L}_\mu(\theta_t) - \mathbb{E}[\mathcal{L}_\mu(\theta_{t+1})] + \frac{L_\mu \eta_t^2 \sigma^2}{2}.
\end{equation}
Summing from $t=0$ to $T-1$ and dividing by $T \eta_t / 2$:
\begin{equation}
    \frac{1}{T} \sum_{t=0}^{T-1} \mathbb{E}[\|\nabla \mathcal{L}_\mu(\theta_t)\|^2] \le \frac{2(\mathcal{L}_\mu(\theta_0) - \mathcal{L}^*)}{\sqrt{T}} + \frac{L_\mu \sigma^2}{\sqrt{T}}.
\end{equation}

\textbf{Step 4: Relating back to $\mathcal{L}$.}
From \Cref{lem:smoothing_props}, $\|\nabla \mathcal{L}_\mu(\theta) - \nabla \mathcal{L}(\theta)\| \le \frac{\mu d L_g}{2}$. Using $\|a+b\|^2 \le 2\|a\|^2 + 2\|b\|^2$:
\begin{equation}
\begin{split}
    \|\nabla \mathcal{L}(\theta_t)\|^2 &\le 2\|\nabla \mathcal{L}_\mu(\theta_t)\|^2 \\ 
    & \quad + 2\|\nabla \mathcal{L}(\theta_t) - \nabla \mathcal{L}_\mu(\theta_t)\|^2 \\
    &\le 2\|\nabla \mathcal{L}_\mu(\theta_t)\|^2 + \frac{\mu^2 d^2 L_g^2}{2}.
\end{split}    
\end{equation}
Substituting the bound for $\nabla \mathcal{L}_\mu$, we obtain the final convergence rate:
\begin{equation}
    \min_{t} \mathbb{E}[\|\nabla \mathcal{L}(\theta_t)\|^2] \le O\left(\frac{1}{\sqrt{T}}\right) + O(\mu^2 d^2).
\end{equation}
This confirms that \method{} converges to a neighborhood of the stationary point, with the radius controlled by the smoothing parameter $\mu$ and dimension $d$.
\end{proof}

\begin{figure*}[t]
    \centering
    \begin{subfigure}[b]{0.48\textwidth}
    \includegraphics[width=\textwidth]{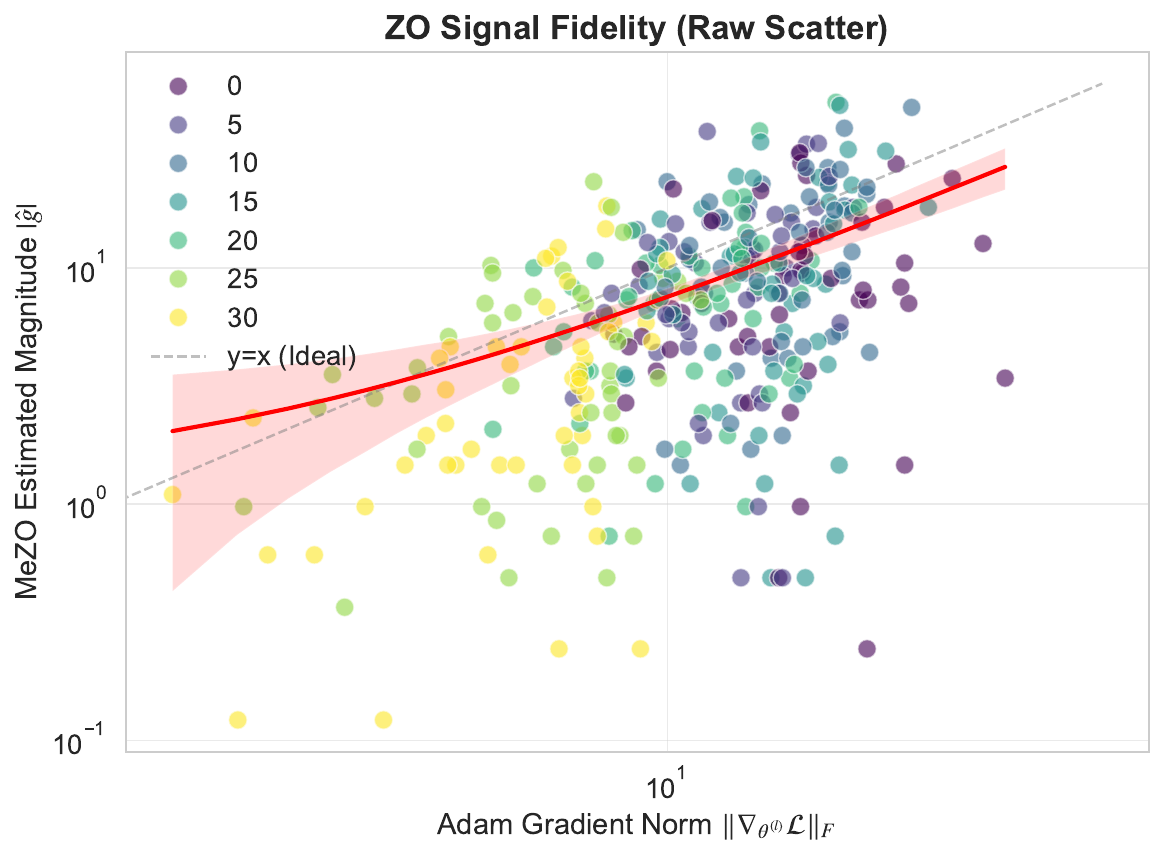}
        \caption{Raw Correlation Scatter}
        \label{fig:fidelity_scatter}
    \end{subfigure}
    \hfill
    \begin{subfigure}[b]{0.48\textwidth}
        \includegraphics[width=\textwidth]{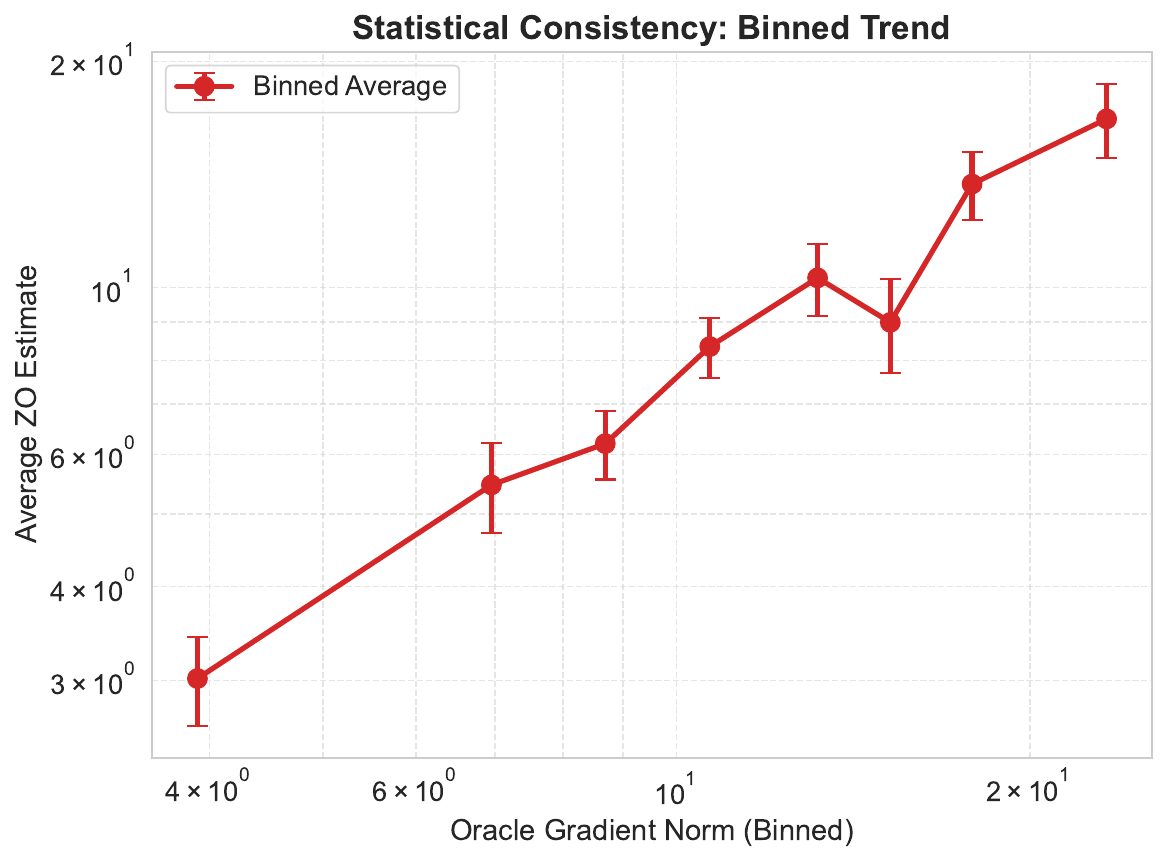}
        \caption{Binned Average Trend}
        \label{fig:fidelity_binned}
    \end{subfigure}
    \caption{
        \textbf{Signal Fidelity Analysis of ZO Estimates.} 
        We investigate whether the noisy Zeroth-Order estimates can serve as a valid proxy for layer sensitivity.
        \textbf{(a)} The scatter plot shows the ZO estimate magnitude $|\hat{g}|$ versus the Oracle gradient norm $\|\nabla_{\boldsymbol{\theta}^{(l)}} \mathcal{L}\|_F$ for individual layers across steps. Despite the variance inherent to random projection, a positive Spearman correlation ($\rho=0.48$) is observed.
        \textbf{(b)} By binning the oracle norms and averaging the corresponding ZO estimates, a strict monotonic relationship is revealed, confirming that ZO perturbations statistically preserve the relative importance ranking of layers.
        Empirical validation of ZO signal fidelity.
        (Left) A scatter plot comparing the oracle gradient norm $\|\nabla_{\boldsymbol{\theta}^{(l)}} \mathcal{L}\|_F$ (computed via backpropagation) against the zeroth-order estimated magnitude $|\hat{g}|$ for each layer during fine-tuning. Despite the high variance intrinsic to random perturbations, we observe a positive Spearman correlation ($\rho \approx 0.48$), indicating that ZO estimates preserve the relative ranking of layer sensitivity.
        (Right) Statistical consistency analysis using binned data. When layers are grouped by their oracle gradient norms, the average ZO estimate exhibits a strictly monotonic increasing trend with tight error bars (SEM). This confirms that, on expectation, ZO feedback serves as a reliable proxy for identifying sensitive layers, validating the feasibility of using ZO signals as rewards in a Multi-Armed Bandit framework.
    }
    \label{fig:signal_fidelity}
\end{figure*}
\begin{table*}[ht]
\centering
\caption{\textbf{Scalability Analysis on OPT Series.} We report accuracy (\%) across various model scales. \textbf{AdaLeZO} maintains competitive or superior performance compared to MeZO as model size increases. The best result between ZO methods is marked in \textbf{bold}.}
\label{tab:multi_size_result}
\begin{tabular}{lccccccc}
\toprule
\textbf{Method} & \textbf{SST-2} & \textbf{RTE} & \textbf{BoolQ} & \textbf{WSC} & \textbf{WIC} & \textbf{SQuAD} & \textbf{AVG.} \\ 
\midrule
\multicolumn{8}{c}{\textit{\textbf{OPT-13B}}} \\ 
\cmidrule(lr){1-8}
Zero-shot & 58.80 & 59.60 & 59.00 & 38.50 & 55.00 & 46.20 & 52.85 \\
MeZO & 91.45{\scriptsize$\pm$0.47} & 71.36{\scriptsize$\pm$2.80} & \textbf{70.60}{\scriptsize$\pm$3.46} & 56.09{\scriptsize$\pm$9.48} & \textbf{60.92}{\scriptsize$\pm$1.57} & \textbf{84.36}{\scriptsize$\pm$0.58} & 72.46{\scriptsize$\pm$3.06} \\
AdaLeZO & \textbf{92.20}{\scriptsize$\pm$0.41} & \textbf{71.72}{\scriptsize$\pm$0.83} & 70.00{\scriptsize$\pm$0.95} & \textbf{57.37}{\scriptsize$\pm$7.28} & 60.03{\scriptsize$\pm$1.70} & 83.52{\scriptsize$\pm$0.80} & \textbf{72.48}{\scriptsize$\pm$2.00} \\ 
\midrule
\multicolumn{8}{c}{\textit{\textbf{OPT-30B}}} \\ 
\cmidrule(lr){1-8}
Zero-shot & 56.70 & 52.00 & 39.10 & 38.50 & 50.20 & 46.50 & 47.17 \\
MeZO & 89.91{\scriptsize$\pm$0.31} & \textbf{65.46}{\scriptsize$\pm$1.04} & \textbf{68.10}{\scriptsize$\pm$0.87} & 57.10{\scriptsize$\pm$1.81} & 57.37{\scriptsize$\pm$8.18} & \textbf{80.31}{\scriptsize$\pm$0.51} & 69.71{\scriptsize$\pm$2.12} \\
AdaLeZO & \textbf{90.75}{\scriptsize$\pm$0.48} & 64.98{\scriptsize$\pm$0.95} & 67.63{\scriptsize$\pm$0.80} & \textbf{58.01}{\scriptsize$\pm$7.77} & \textbf{57.42}{\scriptsize$\pm$1.26} & 80.19{\scriptsize$\pm$0.98} & \textbf{69.83}{\scriptsize$\pm$\textbf{2.04}} \\ 
\bottomrule
\end{tabular}%
\end{table*}

\begin{table*}[t]
\centering
\caption{Comparison of memory footprint and latency breakdown across different ZO methods on OPT-6.7B. \textbf{Bold} indicates the best performance in each group. Our \textbf{Ada} series achieves significant speedup with negligible overhead.}
\label{tab:latency_breakdown}

\renewcommand{\arraystretch}{1.1} 
\setlength{\tabcolsep}{8pt}

\begin{tabular}{l c ccc c c}
\toprule
\multirow{2}{*}{\textbf{Method}} & \textbf{Peak Mem.} & \multicolumn{3}{c}{\textbf{Latency Breakdown (s)}} & \textbf{Total Time} & \multirow{2}{*}{\textbf{SpeedUp}} \\ 
\cmidrule(lr){3-5}
 & (GB) $\downarrow$ & Perturb & Forward & Update & (s) $\downarrow$ & $\uparrow$ \\ 
\midrule

MeZO    & 15.56 & 0.32 & 0.49 & 0.11 & 0.91 & 1.00$\times$ \\
AdaLeZO & \textbf{15.56} & \textbf{0.04} & 0.49 & \textbf{0.01} & \textbf{0.53} & \textbf{1.70$\times$} \\ 
\addlinespace 

LOZO    & 15.58 & 0.27 & 0.49 & 0.09 & 0.85 & 1.00$\times$ \\
AdaLoZO & \textbf{15.58} & \textbf{0.03} & 0.49 & \textbf{0.01} & \textbf{0.53} & \textbf{1.59$\times$} \\ 
\addlinespace

DiZO    & 27.96 & 0.31 & 0.49 & 0.11 & 0.92 & 1.00$\times$ \\
AdaDiZO & \textbf{27.96} & \textbf{0.03} & 0.49 & \textbf{0.01} & \textbf{0.54} & \textbf{1.71$\times$} \\ 
\addlinespace

HiZOO    & 27.96 & 0.60 & 0.75 & 0.37 & 1.73 & 1.00$\times$ \\
AdaHiZOO & \textbf{27.96} & \textbf{0.06} & 0.76 & \textbf{0.03} & \textbf{0.85} & \textbf{2.03$\times$} \\ 
\addlinespace

PseuZO & 16.71 & 0.21 & 0.53 & 1.34 & 2.08 & 1.00$\times$ \\
AdaPZO & \textbf{16.71} & \textbf{0.02} & 0.54 & \textbf{0.15} & \textbf{0.70} & \textbf{2.97$\times$} \\ 
\bottomrule
\end{tabular}
\end{table*}

\begin{table*}[ht]
\centering
\caption{Detailed comparison of per-step training latency (in seconds) and speedup ratios across tasks with varying sequence lengths. $T_{fwd}$ denotes forward pass time, and $T_{ovh}$ denotes overhead (perturbation + update).}
\label{tab:speedup_details}
\begin{tabular}{lc|ccc|ccc|c}
\toprule
\multirow{2}{*}{\textbf{Task}} & \multirow{2}{*}{\textbf{Length}} & \multicolumn{3}{c|}{\textbf{MeZO}} & \multicolumn{3}{c|}{\textbf{AdaLeZO (Ours)}} & \multirow{2}{*}{\textbf{Speedup}} \\
 &  & $T_{fwd}$ & $T_{ovh}$ & \textbf{Total} & $T_{fwd}$ & $T_{ovh}$ & \textbf{Total} &  \\ \midrule
Copa/SST2 & 15  & 0.047 & 0.421 & 0.469 & 0.048 & 0.044 & \textbf{0.091} & \textbf{5.13$\times$} \\
WIC       & 42  & 0.095 & 0.421 & 0.517 & 0.097 & 0.047 & \textbf{0.144} & \textbf{3.60$\times$} \\
WSC       & 51  & 0.114 & 0.423 & 0.537 & 0.116 & 0.048 & \textbf{0.164} & \textbf{3.28$\times$} \\
RTE       & 84  & 0.177 & 0.420 & 0.598 & 0.178 & 0.045 & \textbf{0.224} & \textbf{2.67$\times$} \\
CB        & 91  & 0.198 & 0.425 & 0.623 & 0.199 & 0.046 & \textbf{0.245} & \textbf{2.54$\times$} \\
BoolQ     & 132 & 0.283 & 0.421 & 0.704 & 0.286 & 0.045 & \textbf{0.331} & \textbf{2.13$\times$} \\
SQuAD     & 188 & 0.392 & 0.421 & 0.813 & 0.392 & 0.042 & \textbf{0.434} & \textbf{1.87$\times$} \\
ReCoRD    & 247 & 0.490 & 0.425 & 0.915 & 0.498 & 0.044 & \textbf{0.542} & \textbf{1.69$\times$} \\
DROP      & 307 & 0.622 & 0.425 & 1.048 & 0.620 & 0.043 & \textbf{0.663} & \textbf{1.58$\times$} \\
MultiRC   & 373 & 0.730 & 0.425 & 1.155 & 0.741 & 0.042 & \textbf{0.783} & \textbf{1.48$\times$} \\
\bottomrule
\end{tabular}%
\end{table*}

\begin{algorithm*}[t]
\caption{AdaLeZO: Adaptive Layer-wise Zeroth-Order Optimization}
\label{alg:adalezo}
\begin{algorithmic}[1]
\REQUIRE Model parameters $\theta \in \mathbb{R}^d$ partitioned into $L$ layers $\{\theta^{(1)}, \dots, \theta^{(L)}\}$; \\
Learning rate $\eta$; Perturbation scale $\mu$; Sampling ratio $\rho$; \\
\textcolor{gray}{// Bandit Hyperparameters:} Temperature $\tau$, EMA factor $\alpha$, Exploration $\gamma$, IPW Clipping $C$.

\STATE \textbf{Initialize:} Reward estimates $Q \gets \mathbf{0} \in \mathbb{R}^L$; Time $t \gets 0$.
\STATE \textbf{Initial Sampling:} Run \textsc{ResampleLayers}() to obtain initial $\mathcal{I}, \mathbf{n}, \mathbf{p}$.

\FOR{$t = 0, 1, \dots, T-1$}
    \STATE Sample random seed $s_t$.
    
    \STATE \textcolor{gray}{\textit{// Phase 1: Sparse Perturbation \& Loss Evaluation}}
    \STATE $\theta \gets \textsc{Perturb}(\theta, \mathcal{I}, \mu, s_t)$ \COMMENT{\textit{*Perturbation}}
    \STATE $\mathcal{L}_+ \gets \mathcal{L}(\theta)$ \COMMENT{\textit{*Forward Pass}}
    
    \STATE $\theta \gets \textsc{Perturb}(\theta, \mathcal{I}, -2\mu, s_t)$ \COMMENT{\textit{*Perturbation}}
    \STATE $\mathcal{L}_- \gets \mathcal{L}(\theta)$ \COMMENT{\textit{*Forward Pass}}
    
    \STATE $\theta \gets \textsc{Perturb}(\theta, \mathcal{I}, \mu, s_t)$ \COMMENT{\textit{*Perturbation}}
    
    \STATE Estimate scalar gradient proxy: $\hat{g}_{\text{scalar}} \gets \frac{\mathcal{L}_+ - \mathcal{L}_-}{2\mu}$

    \STATE \textcolor{gray}{\textit{// Phase 2: Count-Aware Sparse Update}}
    \FOR{each active layer $l \in \mathcal{I}$}
        \STATE Set seed $(s_t + l)$ and sample noise $z^{(l)} \sim \mathcal{N}(\mathbf{0}, \mathbf{I}_{d_l})$.
        
        \STATE \textcolor{gray}{\textit{// Calculate Adaptive Weight}}
        \STATE $w_l \gets \min\left(\frac{1}{K \cdot p_l}, C\right)$ \COMMENT{Clipped IPW weight}
        \STATE $n_l \gets \mathbf{n}[l]$ \COMMENT{Multiplicity from sampling with replacement}
        
        \STATE \textcolor{gray}{\textit{// Update Layer Parameters}}
        \STATE $\hat{g}^{(l)} \gets \hat{g}_{\text{scalar}} \cdot w_l \cdot n_l \cdot z^{(l)}$ 
        \STATE $\theta^{(l)} \gets \theta^{(l)} - \eta \cdot \hat{g}^{(l)}$ \COMMENT{\textit{*Update}}
        
        \STATE \textcolor{gray}{\textit{// Update Bandit Estimates}}
        \STATE $Q_l \gets (1-\alpha)Q_l + \alpha |\hat{g}_{\text{scalar}}|$ \COMMENT{Update reward tracking}
    \ENDFOR
    
    \STATE \textcolor{gray}{\textit{// Phase 3: Adaptive Re-sampling for Next Step}}
    \STATE \textsc{ResampleLayers}()
\ENDFOR

\vspace{0.2cm}
\hrule
\vspace{0.1cm}

\STATE \textbf{Procedure} \textsc{Perturb}($\theta, \mathcal{I}, \delta, s$):
\FOR{layer $l \in \mathcal{I}$}
    \STATE Set seed $(s + l)$; Sample $z^{(l)} \sim \mathcal{N}(\mathbf{0}, \mathbf{I}_{d_l})$
    \STATE $\theta^{(l)} \gets \theta^{(l)} + \delta \cdot z^{(l)}$
\ENDFOR
\STATE \textbf{return} $\theta$

\vspace{0.2cm}
\hrule
\vspace{0.1cm}

\STATE \textbf{Procedure} \textsc{ResampleLayers}():
\STATE \quad $K \gets \max(1, \lfloor \rho L \rfloor)$
\STATE \quad \textcolor{gray}{// Compute Sampling Probabilities}
\STATE \quad $\mathbf{p}_{\text{soft}} \gets \text{Softmax}(Q / \tau)$
\STATE \quad $\mathbf{p} \gets (1-\gamma)\mathbf{p}_{\text{soft}} + \gamma/L$ \COMMENT{Mix with uniform exploration}
\STATE \quad \textcolor{gray}{// Sampling with Replacement}
\STATE \quad Sample indices $\mathcal{S} \sim \text{Multinomial}(K, \mathbf{p})$
\STATE \quad $\mathcal{I} \gets \text{Unique}(\mathcal{S})$ \COMMENT{Set of unique active layers}
\STATE \quad $\mathbf{n} \gets \text{CountFrequencies}(\mathcal{S})$ \COMMENT{Calculate multiplicity $n_l$ for $l \in \mathcal{I}$}
\end{algorithmic}
\end{algorithm*}

\end{document}